%% file: main.tex
\newcommand{\tabincell}[2]{\begin{tabular}{@{}#1@{}}#2\end{tabular}}
\documentclass[conference]{IEEEtran}
\input{tex/mathcmd}

\IEEEoverridecommandlockouts
\usepackage{cite}
\usepackage{algorithmic}
\usepackage{graphicx}
\usepackage{textcomp}
\usepackage{booktabs}

\def\BibTeX{{\rm B\kern-.05em{\sc i\kern-.025em b}\kern-.08em
    T\kern-.1667em\lower.7ex\hbox{E}\kern-.125emX}}
\usepackage{multirow}
\usepackage{subfigure}
\usepackage{color}

\title{Learning Robust Representations with Graph Denoising Policy Network}

\begin{document}

\author{
Lu Wang$^1$, Wenchao Yu$^{2*}$, Wei Wang$^3$, Wei Cheng$^2$,  Wei Zhang$^1$, Hongyuan Zha$^4$, Xiaofeng He$^{1*}$, Haifeng Chen$^2$\\
$^1$East China Normal University, $^2$NEC Laboratories America, Inc.,\\ $^3$University of California Los Angeles, $^4$Georgia Institute of Technology
\\
\{luwang,xfhe\}@stu.ecnu.edu.cn, \{wyu,weicheng,haifeng\}@nec-labs.com,\\weiwang@cs.ucla.edu, zhangwei.thu2011@gmail.com,	zha@cc.gatech.edu 
\thanks{$^*$Corresponding authors.}
}
\maketitle


\input{tex/abstract.tex}


\input{tex/introduction.tex}
\input{tex/preliminary.tex}
\input{tex/approach.tex}
\input{tex/submodular.tex}
\input{tex/results.tex}
\input{tex/related.tex}
\input{tex/conclusion.tex}
\bibliographystyle{ieeetr}
\bibliography{references}
\end{document}

%% file: tex/mathcmd.tex
\usepackage{amsmath,amssymb,amsthm}

\newcommand{\G}{\mathcal{G}}

\newcommand{\N}{\mathcal{N}}
\newcommand{\R}{\mathbb{R}}

\newcommand{\E}{\mathbb{E}}

\newcommand{\V}{{\mathcal{V}}}

\newtheorem{prop}{Proposition}

\newtheorem{theorem}{Theorem}

\newtheorem{definition}{Definition}

\newlength\myindent

%% file: tex/abstract.tex
\begin{abstract}
Graph representation learning, aiming to learn low-dimensional representations which capture the geometric dependencies between nodes in the original graph, has gained increasing popularity in a variety of graph analysis tasks, including node classification and link prediction. 
Existing representation learning methods based on graph neural networks and their variants rely on the aggregation of neighborhood information, which makes it sensitive to noises in the graph,
e.g. erroneous links between nodes, incorrect/missing node features.
In this paper, we propose Graph Denoising Policy Network (short for GDPNet) to learn robust representations from noisy graph data through reinforcement learning. GDPNet first selects \emph{signal} neighborhoods for each node, and then aggregates the information from the selected neighborhoods to learn node representations for the down-stream tasks. Specifically, in the \emph{signal neighborhood selection} phase, GDPNet optimizes the neighborhood for each target node by formulating the process of removing noisy neighborhoods as a Markov decision process and learning a policy with task-specific rewards received from the \emph{representation learning} phase. In the \emph{representation learning} phase, GDPNet aggregates features from signal neighbors to generate node representations for down-stream tasks, and provides task-specific rewards to the \emph{signal neighbor selection} phase. These two phases are jointly trained to select optimal sets of neighbors for target nodes with maximum cumulative task-specific rewards, and to learn robust representations for nodes. 
Note that GDPNet is naturally an \emph{inductive} model which can leverage both graph structure and the associated node feature information to efficiently generate representations for unseen nodes.
Experimental results on node classification task demonstrate the effectiveness of GDNet, outperforming the state-of-the-art graph representation learning methods on several well-studied datasets. Additionally, we show that, with a carefully designed reward function, GDPNet is mathematically equivalent to solving the submodular maximizing problem, which theoretically guarantees the best approximation to the optimal solution with GDPNet.

\end{abstract}

%% file: tex/introduction.tex
\section{Introduction}
Recently, remarkable progress has been made toward graph representation learning, a.k.a graph/network embedding, which solves the graph analytics problem by mapping nodes in a graph to low-dimensional vector representations while effectively preserving the graph structure~\cite{hamilton2017representation,cai2018comprehensive,cui2018survey,yu2018learning}. Graph neural networks (GNNs) have been widely applied in graph analysis due to the ground-breaking performance with deep architectures and recent advances in optimization techniques~\cite{scarselli2008graph,zhou2018graph}. Existing representation learning methods based on GNNs, e.g. GraphSAGE~\cite{hamilton2017inductive}, Graph Convolution Networks (GCNs)~\cite{kipf2016semi,chen2018fastgcn} and Graph Attention Networks (GATs)~\cite{velivckovic2017graph}, rely on the aggregation of neighborhood information, which makes the model vulnerable to noises in the input graph. 

Some examples of such noises are as follows:
\begin{itemize}
	\item In knowledge graphs or open information extraction systems, spurious information may produce erroneous links between nodes. Likewise, incomplete information may lead to missing links.  
	\item In task-driven graph analysis, mislabeled samples, or cross-class links can be viewed as noises in node classification task.
	\item Node features such as user profiles in social networks are often missing, or filled with obsolete or incorrect values. 
\end{itemize}

\begin{figure}
\centerline{\includegraphics[width=.88\linewidth]{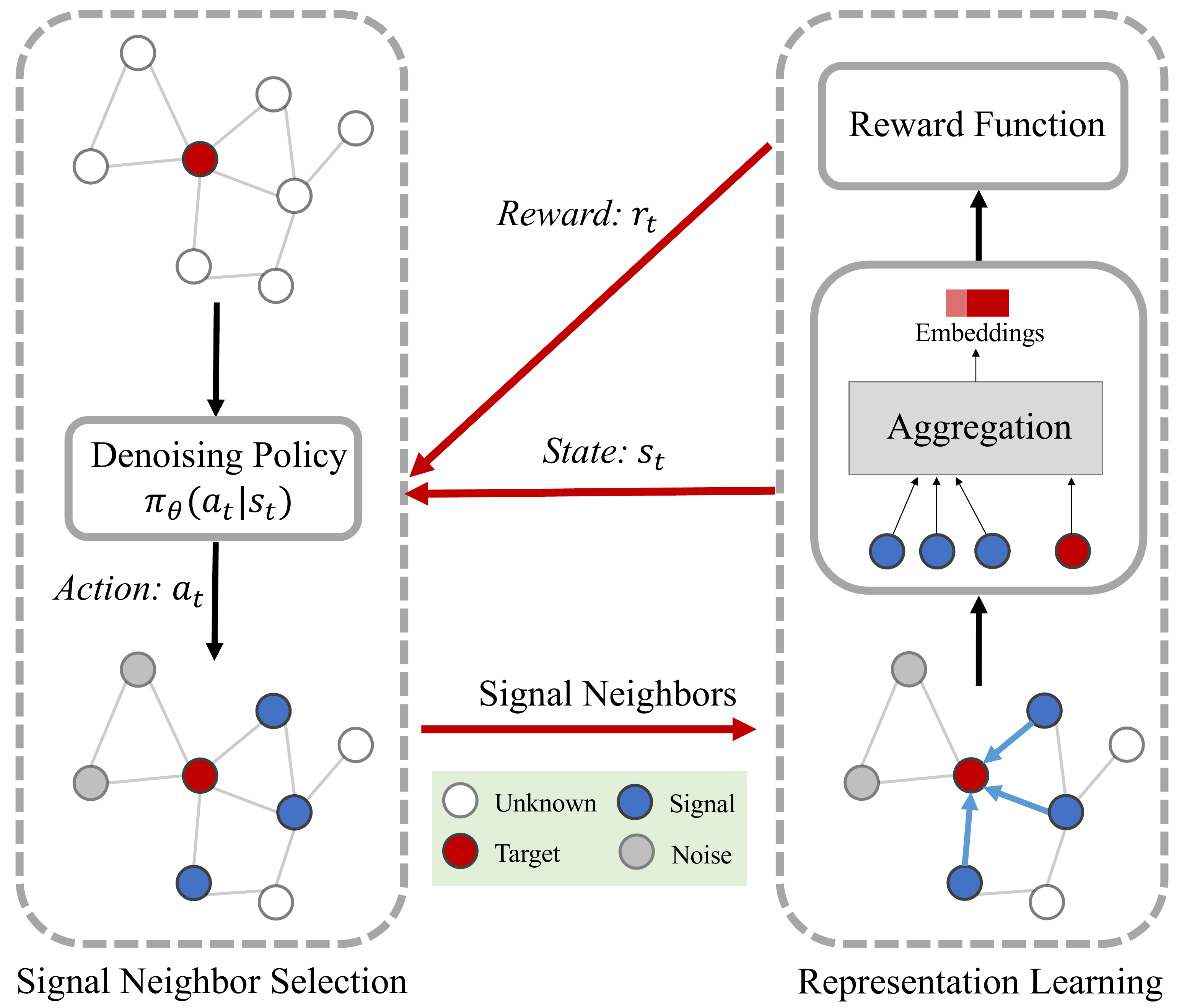}}
\caption{The framework of GDPNet with one-hop neighborhood aggregation}
\label{fig:model}
\end{figure}

\begin{figure*}
\centerline{\includegraphics[width=.95\linewidth]{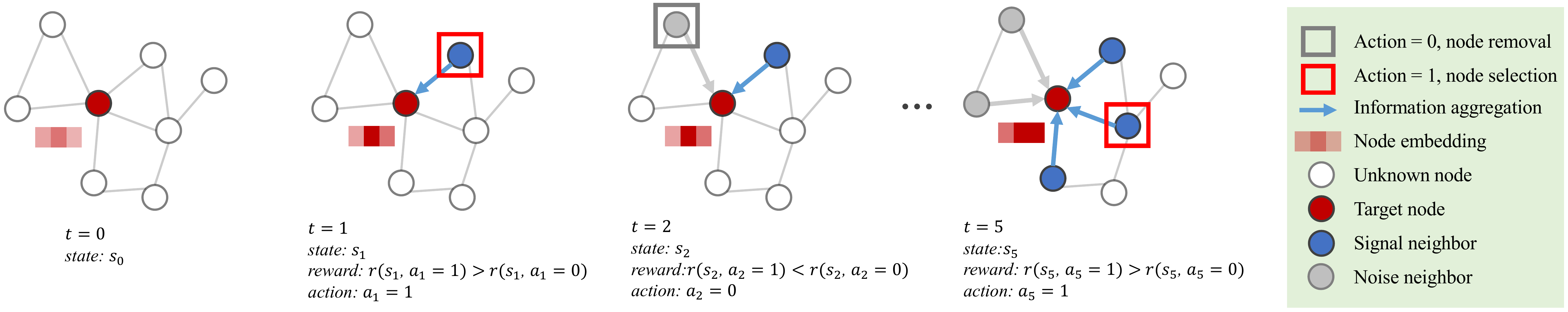}}
\caption{Illustration of the GDPNet model from the view of signal neighbor selection}
\label{fig:demo}
\end{figure*}

Good graph representations are expected to be robust to the  erroneous links, mislabeled nodes and partial corrupted features in the input graph, and capture geometric dependencies among nodes in the graph. However existing approaches have limited efforts on robustness study in this regard. 
In order to overcome this limitation of graph representation learning in handling noisy graph data, we propose Graph Denoising Policy Network, denoted as GDPNet, to learn robust representations through reinforcement learning. GDPNet includes two phases: \emph{signal neighbor selection} and \emph{representation learning}. It first selects \emph{signal} neighbors for each node, and then aggregates the information from the selected neighbors to learn node representations with respect to the down-stream tasks. 

The major challenge here is on how to train these two phases jointly, particularly when the model has no explicit knowledge about where the noise might be. 
We address this challenge by formulating the graph denoising process as a Markov decision process. Intuitively, although we do not have an explicit supervision for the signal neighbor selection, we can measure the performance of the representations learned with the selected neighbors on tasks like node classification, then the task-specific rewards received from the representation learning phase can be used for trial-and-error-search. In the \emph{signal neighbor selection} phase, as shown in Fig.~\ref{fig:demo}, GDPNet optimizes the neighborhood for each node by formulating the process of removing the noisy neighbors as a Markov decision process and learning a policy with the task-specific rewards received from the \emph{representation learning} phase.
In the \emph{representation learning} phase, GDPNet aggregates features from signal neighbors to generate node representations for down-stream tasks, and provides task-specific rewards to the \emph{signal neighbor selection} phase. 
In the \emph{representation learning} phase, GDPNet trains a set of aggregator functions that accumulate feature information from the selected signal neighbors of each target node. Thus in the test time, the representations of unseen nodes can be generated with the trained GDPNet with graph structure and the associated node feature information. The task-specific rewards computed w.r.t the down-stream tasks are passed to the \emph{signal neighbor selection} phase. These two phases are jointly trained to select optimal sets of neighbors for target nodes with maximum cumulative task-specific rewards, and to learn robust representations for nodes. 

We evaluate GDPNet on node classification benchmark, which tests GDPNet's ability to generate useful representations on unseen data. Experimental results show that GDPNet outperforms state-of-the-art graph representation learning baselines on several well-studied datasets by a large margin, which demonstrates the effectiveness of our approach.  
In summary, our contributions in this work include:
\begin{itemize}
	\item We propose a novel model, GDPNet, for robust graph representation learning through reinforcement learning. GDPNet consists of two phrases, namely \emph{signal neighbor selection} and \emph{representation learning}, which enables GDPNet to effectively learn node representations from noisy graph data.
	\item We formulate signal neighbor selection as a reinforcement learning problem, which enables the model to perform graph denoising just with weak supervision from the task-specific reward signals.
	\item GDPNet is able to generate representations for unseen nodes in an \emph{inductive} fashion, which leverages both graph structure and the associated node feature information.
	
	\item GDPNet is proved to be mathematically equivalent to solving the submodular maximizing problem, which guarantees our model can be bounded w.r.t the optimal solution.
\end{itemize}

The rest of this work is organized as follows. Section~\ref{sec:pre} reviews the preliminaries of graph neural networks and reinforcement learning. Section~\ref{sec:method} formally defines the graph representation learning problem through reinforcement learning, along with the description of GDPNet. Section~\ref{sec:submodular} discusses the relations of GDPNet to submodular maximization problem. Section~\ref{sec:results} evaluates GDPNet on node classification tasks using various real-world graph data. Section~\ref{sec:related} briefly surveys related work on graph representation learning and reinforcement learning on graph, followed by the conclusion in Section~\ref{sec:con}.

%% file: tex/preliminary.tex
\section{Preliminaries}\label{sec:pre}

\subsection{Graph Neural Network}
Graph neural networks (GNNs) are utilized to encode nodes in a low-dimensional space so that similarity in this space approximates similarity in the original graph~\cite{zhou2018graph}. GNNs generate node representations based on local neighborhoods, that is, by aggregating information from their neighbors using neural networks. Mathematically, the basic GNN can be formulated as follows,
\begin{align}
	h_v^0 &= x_v \\
	h_v^k &= \sigma\left(W_k \sum_{u\in \N(v)} \frac{h_u^{k-1}}{|\N(v)|} + B_k h_v^{k-1}\right), \forall k > 0
\end{align}
where $\N(v)$ is the neighborhood set of node $v$, $x_v$ represents node feature vector, $h_v^k$ represents $k^\textit{th}$ layer embedding of node $v$, $\sigma$ is the non-linear activation function (e.g. ReLU or tanh), $W_k$ and $B_k$ are the parameters to be learned.

\subsection{Reinforcement Learning}
The goal of reinforcement learning is to learn a policy which can obtain maximum cumulative reward by making multi-step decisions in a Markov decision process. Policy gradient is the main approach to solve reinforcement learning problems which directly optimizes this policy via calculating the gradient of the cumulative reward and making gradient ascent. Specifically, the policy is modeled as $\pi_\theta(a\vert s)$.
To estimate the parameter $\theta$ of $\pi_\theta$, policy gradient methods maximize the expected cumulative reward
from the start state $s_1$ to the end of the decision process $s_T$. The objective function of reinforcement learning is defined as follows: 
\begin{equation}
 J(\theta) = \mathbb{E}_{\tau\sim \pi_\theta}\left[\sum_{t=1}^Tr_t\right] \approx \frac{1}{N}\sum_{i=1}^N\sum_{t=1}^T r_t  
\end{equation}
where $\tau$ is the trajectories generated by $\pi_\theta$ which consists of state $s_t$, action $a_t$ and reward $r_t$, $N$ is the number of trajectory samples, $T$ is the length of the decision process.
Policy gradient optimizes the parameter $\theta$ via gradient ascent where the gradient $\nabla_\theta J(\theta)$ is calculated by the \emph{policy gradient theorem}~\cite{sutton2000policy}:
\begin{equation}\label{pg}
\nabla_\theta J(\theta) =  \mathbb{E}_{\tau\sim \pi_\theta}\left[\left(\sum_{t=1}^T\nabla_\theta \log\pi_\theta\left(a_t\vert s_t\right)\right)\sum_{t=1}^T r_t\right]
\end{equation}

%% file: tex/approach.tex
\section{Approach} \label{sec:method}




We formulate the robust graph representation learning problem as sequentially selecting an optimal set of neighbors for each node with maximum cumulative reward signals and aggregating features from nodes' optimal neighborhoods. In this part, we formally define the problem, the environment setting for signal neighbor selection, and the GDPNet model.

\subsection{Problem Formulation}
Given an attributed graph $\G=\{\mathcal{E},\mathcal{V}, X\}$, where $\mathcal{E}$ is the edge set and $\V$ is the node set. $X\in \mathbb{R}^{|\mathcal{V}|\times D}$ collects the attribute information for each node where $x_v\in \mathbb{R}^D$ is a $D$-dimensional attribute vector of node $v\in \V$. Note that we can simply use one-hot encoding for node features for a graph without attributes. Given a target node $v$, let $\mathcal{N}(v) = \{u_{1}, u_{2},..., u_{\vert \mathcal{N}(v)\vert}\}$ be the one-hop neighbors of $v$. 

We aim to find a lower-dimensional representation $h_v$ for node $v\in\V$. Firstly, a function $f:2^{\mathcal{N}(v)}\rightarrow 2^{\hat{\mathcal{N}}(v)}$ is learned to map a neighborhood set $\mathcal{N}(v)$ into a signal neighborhood set $\hat{\mathcal{N}}(v)$, where $\hat{\mathcal{N}}(v)\subseteq\mathcal{N}(v)$. Then the node representations are generated based on the signal neighborhood set, $h: 2^{\hat{\mathcal{N}}(v)}\xrightarrow{f} \R^d$ . Given an order of the neighbors $u_1,...,u_{\vert\mathcal{N}(v)\vert}$, we decompose the conditional probability of $\hat{\mathcal{N}}(v)$ given $\mathcal{N}(v)$ as $p(\hat{\mathcal{N}}(v)\vert\mathcal{N}(v))=\Pi^{\vert \hat{\mathcal{N}}(v)\vert}_{t=1}p(a_t\vert\mathcal{N}(v),a_1,...,a_{t-1})$ using chain rule~\cite{liu2015optimality}, where $a_t = \{0,1\}$, $a_t=1$ indicates selecting $u_t$ as a signal neighbor while $a_t=0$ indicates removing $u_t$. We solve this signal neighbor selection problem by learning a policy $\pi_\theta (a_t\vert s_t)= p(a_t\vert\mathcal{N}(v),a_1,...,a_{t-1})$ with neighborhood set $\mathcal{N}(v)$ and the predicted action values $\{a_i\}_{i=1}^{t-1}$ as inputs. The objective of signal neighbor selection is to select a subset of neighbors that maximize a given reward function $R_\pi(\hat{\mathcal{N}}(v)) = \mathbb{E}_{\hat{\mathcal{N}}(v)}\left[\sum_{t=1}^{\vert\hat{\mathcal{N}}(v) \vert}r_t\right] $, where $\hat{\mathcal{N}}(v)$ is the generated signal neighborhood set, $r_t$ is the task-specific reward used to evaluate the action $a_t$, and $R_\pi$ is the cumulative reward function. The representation of node $v$ can then be learned by aggregating the neighborhood information from the signal neighbors $\hat{\mathcal{N}}(v)$.

Selecting an optimal subset from a candidate set by maximizing an objective function is NP-hard which can be approximatively solved by greedy algorithms with a submodular function~\cite{nemhauser1978analysis}. With this observation, we design our reward function that satisfies submodularity, and show that the proposed GDPNet is mathematically equivalent to solving the submodular maximizing problem. Thus our solution can be bounded by $(1-\frac{1}{e}) R\left(\mathcal{N}(v)^*\right)$, where $\mathcal{N}(v)^*$ is the optimal neighborhood set.

\begin{table}
\centering
\caption{Notation description}
\begin{tabular}{cl}\toprule
Notation & Description \\\midrule
$\mathcal{N}(v)$ & neighborhood set of node $v$ \\
$\hat{\mathcal{N}}(v)$ & signal neighborhood set of node $v$  \\
$\hat{\mathcal{N}}(v)_t$ & signal neighborhood set of node $v$ at time $t$ \\
$\hat{\mathcal{N}}(v)_t^\mathsf{c}$ & complementary set, $\hat{\mathcal{N}}(v)_t^\mathsf{c} = \mathcal{N}(v)\setminus \hat{\mathcal{N}}(v)_t$\\\midrule
$h_v^t$& $h_v^t\in\R^d$, embedding of target node $v$ at time $t$\\
$x_v$& $x_v\in\R^D$, feature vector of node $v$ \\
\midrule
$s_t = [h_v^t,h_{u_t}]$ & the states, $s_{t}\in\mathbb{R}^{2\times d}$, $u_t$ is the $t^\textit{th}$ neighbor of $v$ \\
$r_t $ & reward function at time $t$\\
$R_\pi $ & total reward function, $R_\pi  =\E_{\hat{\mathcal{N}}(v)}\left[\sum_{t=1}^{\vert\hat{\mathcal{N}}(v) \vert}r_t\right]$\\
$a_t = \{0,1\}$ & \tabincell{l}{action space, where $a_t=1$ represents neighbor \\ selection, and $a_t=0$ represents neighbor removal} \\
$\pi_\theta(a_t\vert s_t)$ & \tabincell{l}{the policy which maps the current state into the \\ action distribution.}\\\bottomrule
\end{tabular}
\label{tab:notation}
\end{table}

\subsection{Signal Neighbor Selection Environment}\label{subsec:env}

We formulate the problem of selecting a set of signal neighbors from a given neighborhood set as a Markov decision process (MDP) $(\mathcal{S}, \mathcal{A},P,R,\gamma)$, where $\mathcal{S}$ is the state space, $\mathcal{A}$ is the action space, $P$ is the state transition probability matrix that describes the transition probability of the state after taking an action, $R$ is the reward function and $\gamma$ is discount factor of the MDP. The signal neighbor selection process can be described by a trajectory with $\hat{\mathcal{N}}(v)$ time steps $s_0,a_0,r_0,...,s_{\vert\hat{\mathcal{N}}(v)\vert},a_{\vert\hat{\mathcal{N}}(v)\vert},r_{\vert\hat{\mathcal{N}}(v)\vert}$. MDP requires the state transition dynamics to satisfy the Markov property $p(s_{t+1}\vert s_t)$. Thus we learn a policy $\pi_\theta(a_t\vert s_t)$ that only considers the current state $s_t$.

In reinforcement learning, the agent learns a policy via interacting with the environment. The main components (i.e., state, action, and reward) in the signal neighbor selection environment are described as follows,

\begin{itemize}
\item \textbf{State ($\mathcal{S}$)}: The state $s_t = [h_v^t, h_{u_t}]$ encodes the information from the current node $v$ and the selected node $u_t$, which is concatenation of the intermediate embeddings $h_v^t$ and $h_{u_t}$ of the target node $v$ and the $t^\textit{th}$ neighbor $u_t$, respectively. The calculation of $h_v^t$ and $h_{u_t}$ are defined in Section~\ref{subsec:model}. Consequently, a newly selected neighbor $u_{t}$ will update the embedding of $v$ from $h_v^t$ to $h_v^{t+1}$ which can be viewed as state transition. 
\item \textbf{Action ($\mathcal{A}$)}: Given an order of the neighbors $u_1,...,u_{\vert\mathcal{N}(v)\vert}$ of node $v$, the policy $\pi_\theta(a_t\vert s_t)$ maps the state $s_t$ into an action $a_t =\{0,1\}$ at each time step $t$, $t=1,...,\vert\hat{\mathcal{N}}(v)\vert$.  $a_1 = 1$ indicates $u_1$ is selected as a signal neighbor, while $a_1 = 0$ means $u_1$ is not selected.
\item \textbf{Reward ($R$)}: Our goal is to find an optimal set of signal neighbors $\hat{\mathcal{N}}(v)$ from a finite neighborhood set $\mathcal{N}(v)$ to learn robust graph embedding for downstream tasks such as node classification, link prediction and node clustering. The downstream tasks can produce task-specific scores as the reward signal for the signal neighbor selection phase. To ensure that the combination of the selected neighbors have maximum cumulative rewards. We employ the submodular function framework to define the marginal value reward function: 	
\begin{equation}\label{reward}
r_t = \frac{f_c(\textsc{Agg}(x_v,\{x_{u_t}\}))}{\sum_{\tilde{u}\in\hat{\mathcal{N}}(v)_{t}}f_c(\textsc{Agg}(x_v,\{x_{\tilde{u}}\}))}
\end{equation}
where $\textsc{Agg}(\cdot)$ aggregates both the target node feature $x_v$ and the neighbors' features $\{x_{u_t}\}$ to update the representations of the target node~\cite{hamilton2017inductive}, and $f_c(\cdot)$ returns the micro-averaged F1 score from the node classification task when considers $u_t$ as the neighbor.
\end{itemize}

The environment updates the states from $s_t = [h_v^t,h_{u_t}]$ to $s_{t+1} = [h_v^{t+1},h_{u_{t+1}}]$ by calculating the representations $h_v^{t+1} = \textsc{Agg}(x_{v},\{x_{\tilde{u}},\forall \tilde{u} \in \hat{\mathcal{N}}(v)_t\})$  at time $t+1$. It can be considered as a state transition: 
\begin{align}
p(s_{t+1}\vert s_t) = \sum_{a_t}\pi_\theta(a_t|s_t)p(s_{t+1}\vert s_t,a_t)
\end{align}
If $a_t = 1$, $\hat{\mathcal{N}}(v)_{t} = \hat{\mathcal{N}}(v)_{t-1}\cup \{u_t\}$, otherwise $\hat{\mathcal{N}}(v)_{t} = \hat{\mathcal{N}}(v)_{t-1}$.

\subsection{Graph Denoising Policy Network}\label{subsec:model}

With the definitions of the signal neighbor selection environment, we introduce the GDPNet model which includes two phases: signal neighbor selection and representation learning. Given a target node $v$, GDPNet first takes its neighborhood set $\mathcal{N}(v)$ as input and outputs a signal neighborhood subset $\hat{\mathcal{N}}(v)$. Then the representations $h_v$ is learned by aggregating the information from the signal neighborhood subset $\hat{\mathcal{N}}(v)$.

\subsubsection{Determine the Neighborhood Order}

As aforementioned, we use chain rule to decompose the signal neighbor selection as a sequential decision making process. However, it requires an order to make decisions. Here we design a high-level policy to learn an order $[u_1,...,u_{\vert\mathcal{N}(v)\vert}]$ for the policy $\pi_{\theta}$ to take action.

We define a \emph{regret score}  $l$ for each neighbor to help determine the order. A neighbor with large regret score indicates it will be selected with higher probability. At each time step, we calculate the regret score of each neighbor and sample one of the neighbor to be the $t^\textit{th}$ neighbor. The regret score is described as follows:
\begin{equation}\label{eq:regret}
l_k= W_1\cdot \textrm{ReLU}(W_2\cdot s_t ), s_t = [h_v^t,h_{u_k}]
\end{equation}
where $u_k$ is the $k$-th neighborhood in the neighborhood set $\mathcal{N}(v)$ with a random order and $W_1,W_2$ are parameter matrices. To reduce the size of $\hat{\mathcal{N}}(v)$ for computational efficiency, we add an ending neighbor $u_e$ to $\mathcal{N}(v)$ for early stopping purpose. When $u_e$ is sampled, the neighborhood selection process of node $v$ stops. We use the \emph{Softmax} function to normalize the regret scores, and sample one neighbor from the distribution generated by \emph{Softmax} to be the $t^\textit{th}$ neighbor.
\begin{equation}\label{softmax}
u_t\sim \textsc{Softmax}([l_1,l_2,...,l_e,...,l_{\vert \mathcal{N}(v)_t^\mathsf{c}\vert}])
\end{equation}
where $u_t \in \mathcal{N}(v)_t^\mathsf{c}$ is the $t^\textit{th}$ neighbor for signal neighbor selection, $\mathcal{N}(v)_t^\mathsf{c} = (\mathcal{N}(v)\setminus \hat{\mathcal{N}}(v)_t)$. $l_e$ indicates the regret score of the ending neighbor $u_e$. After selecting a neighbor $u_t$, we adopt the policy $\pi_\theta$ to determine whether to select $u_t$ as a signal neighbor. Then $u_t$ will be removed from $\hat{\mathcal{N}}(v)_t^\mathsf{c}$.

\subsubsection{Signal Neighbor Selection} 

Given the $t^\textit{th}$ neighbor $u_t$, GDPNet takes an action $a_t =\{0,1\}$ at time step $t$ to decide whether to select the $u_t$. We will make $\vert\hat{\mathcal{N}}(v)\vert$ decisions to select the signal neighbors for node $v$. Here the total number of signal neighbors can be automatically determined. As illustrated in Fig.~\ref{fig:demo}, a policy $\pi_\theta(a_t|s_t)$ is learned to map the state $s_t$ to the action $a_t$ at time step $t, t=1,...,\vert\hat{\mathcal{N}}(v)\vert$, meanwhile the corresponding reward $r_t$ will be provided. Our goal is to maximize the total reward of all the actions taken during these time steps, which can be learned by the following policy network,
\begin{align}\label{pi}
\pi_\theta(a_t\vert s_t) & = \sigma\left(W_1\cdot\textrm{ReLU}(W_2\cdot s_t )\right)\nonumber\\
a_t & \sim \pi_\theta \in\{0,1\}
\end{align}
where $W_1$ and $W_2$ are weight matrices shared with Eq.~(\ref{eq:regret}), and action $a_t$ is sampled from a Bernoulli distribution which is generated by $\pi_\theta(a_t|s_t)$.

\subsubsection{Representation Learning} 

At each time step, GDPNet calculates the embeddings of the target node $v$ and the $t$-th neighbor $u_t$ as follows,
\begin{align}
h_v^t& \leftarrow \textsc{Agg}(x_v,\{ x_{\tilde{u}},\forall \tilde{u} \in \hat{\mathcal{N}}(v)_t\})\\
h_{u_t}& \leftarrow \textsc{Agg}(x_{u_t},\{ \varnothing\})
\end{align}
where $\textsc{Agg}(x,\{ y_i,\forall i \in \mathcal{I}\}) = \sigma(W\cdot \texttt{MEAN}(\{x\}\cup\{ y_i,\forall i \in \mathcal{I}\})$, $x_v$ and $x_{u_t}$ are the features of node $v$ and $u_t$ respectively. We computed the embedding of neighbor $u_t$ via its own feature $x_{u_t}$, because the goal is to evaluate the individual contribution of $u_t$. In this work we only consider one-hop neighbors for simplicity. The GDPNet model can be easily extended to aggregate the information from multi-hop neighbors with an augmented candidate neighborhood set for selecting the signal neighbors.

As defined in Section~\ref{subsec:env}, the state at time step $t$, $s_t = [h_v^t,h_{u_t}]$, is a concatenation of the intermediate node embeddings $h_v^t$ and $h_{u_t}$. Eventually, the representations $h_v$ and state $s_t = [h_v^t,h_{u_t}]$ can be obtained.

\subsubsection{Iteration-wise Optimization} 

We consider an iteration-wise optimization approach to optimize the GDPNet model, which optimizes the signal neighbor selection phrase and representation learning phrase iteratively to learn the policy $\pi_\theta$ and the representations $h_v$. As for representation learning phase, it aggregates the information from the signal neighbors selected by $\pi_\theta$ to learn an embedding $h_v$ for target node $v$. Meanwhile, the policy $\pi_\theta$ is trained with the states calculated by $h_v$ and the corresponding rewards. In this paper, $\pi_\theta$ is optimized with Proximal Policy Optimization (PPO), one of the widely used policy gradient method~\cite{schulman2017proximal}. 
\begin{equation}
\max~~ \mathbb{E}_{s\sim\rho_{\theta_{old}},a\sim q}\left[\frac{\pi_\theta(a|s)}{q(a|s)}Q_{\theta_{old}}(s,a)\right], \quad 
\end{equation}
\begin{displaymath}
s.t. ~~\mathbb{E}_{s\sim\rho_{\theta_{old}}}
\left[D_{KL}(\pi_{\theta_{old}}(\cdot|s)
\parallel \pi_\theta(\cdot|s))\right] \leq \delta
\end{displaymath}
where Kullback–Leibler (KL) divergence penalty is used to control the change of the policy at each iteration to perform a trust region update with a threshold $\delta$. $q(a|s)$ and $Q_{old} = \sum_{i=t}^T \gamma r_i$ are the policy and Q-value, respectively, which are saved before the current time step during training. $\rho_{\theta_{old}}$ is the discounted state distribution defined as,
\begin{align}
\rho_{\theta_{old}}(s_t) = \sum_{t=0}^T \gamma^{t-1} p(s_t=s\vert \pi_{\theta_{old}})
\end{align}

%% file: tex/submodular.tex
\section{Connection with Submodular Maximization}\label{sec:submodular}

The design of the reward function in GDPNet described in Section~\ref{subsec:env} is inspired by the submodular function. With this carefully designed reward function, we build the connections with submodular maximization problem, and show that the solution provided by GDPNet can be bounded by $(1-\frac{1}{e}) R(\N(v)^*)$, where $\N(v)^*$ is the optimal neighborhood set. In this section we first introduce the key definitions related to submodular functions, followed by the proof of monotonicity and submodularity properties of the reward function in GDPNet.

\subsection{Submodular Reward Function}

In this section, we show that given a special form of reward function, the total reward function in GDPNet turns out to be submodular.

\begin{definition}[Submodular Function]
Let $N$ be a finite ground set and $f_s: 2^N \rightarrow \mathbb{R}$ is set function. A set function is submodular if it satisfies the diminishing returns property: $f_s(A\cup{c})-f_s(A) \geq f_s(B\cup{c})-f_s(B)$ and the monotone property: $f_s(A)\leq f_s(B)$ for all $A\subseteq B \subseteq \Omega$ and $v\in \mathcal{N}\setminus B$. 
\end{definition}

\begin{definition}[Submodular Maximization]
Let $w$ be an optimizer which maps a set $\Omega$ in to a subset $A$ with size smaller than $K$:
\begin{align}
w(\Omega) = A, ~~\vert A\vert \leq K
\end{align}

The submodular maximization problem is to find the best possible $w^*$ which satisfying:
\begin{equation}
w^*(\Omega) = {\arg \max}_{A\subseteq  \Omega :\vert A\vert \leq K} f_s(A)
\end{equation}
\end{definition}


The reward function $r_t $ in GDPNet is denoted by $r_t = R_\pi(\hat{\mathcal{N}}(v) \cup \{u_t\} ) - R_\pi(\hat{\mathcal{N}}(v))$, which is also named marginal value. Specifically, the reward function in GDPNet can be expressed as:
\begin{align}\label{eq:reward}
r_t = \frac{f_c(\textsc{Agg}(x_v,\{x_{u_t}\})}{\sum_{\tilde{u}\in\hat{\mathcal{N}}(v)_{t}}f_c(\textsc{Agg}(x_v,\{x_{\tilde{u}}\})}
\end{align}
where $\textsc{Agg}(x,\{y\}) = \sigma(W\cdot \texttt{MEAN}(\{x\}\cup\{ y\})$. Given this reward function $r_t$, we can prove that the cumulative reward function $R_\pi$ is a submodular function.

\begin{prop}
The total reward function $R_\pi$ is a monotone function, where
\begin{align}
R_\pi  =\mathbb{E}_{\hat{\mathcal{N}}(v)}\left[\sum_{t=1}^{\vert\hat{\mathcal{N}}(v) \vert}r_t\right]
\end{align}
\end{prop}
\begin{proof}
$R_\pi$ is monotone if $R_\pi(\hat{\mathcal{N}}(v)_{t_B}) - R_\pi(\hat{\mathcal{N}}(v)_{t_A}) \geq 0 $ whenever $\hat{\mathcal{N}}(v)_{t_A}\subseteq\hat{\mathcal{N}}(v)_{t_B}$ and $t_A < t_B$.
\begin{align}
&R_\pi(\hat{\mathcal{N}}(v)_{t_B}) - R_\pi(\hat{\mathcal{N}}(v)_{t_A}) \\
\Leftrightarrow & \sum_{t=1}^{\vert\hat{\mathcal{N}}(v)_{t_B} \vert}r_t -\sum_{t=1}^{\vert\hat{\mathcal{N}}(v)_{t_A} \vert}r_t = \sum_{t=t_A}^{t_B} r_t
\end{align}
we have $r_t \geq 0$ based on Eq.~(\ref{eq:reward}). Therefore $\sum_{t=t_A}^{t_B} r_t\geq 0$, and $R_\pi$ is monotone.
\end{proof}

\begin{prop}
The total reward function $R_\pi$ satisfies the submodularity property. That is, 
\begin{align}
&R_\pi(\hat{\mathcal{N}}(v)_{t_A} \cup \{u_t\} ) - R_\pi(\hat{\mathcal{N}}(v)_{t_A} ) \nonumber \\
&\geq (R_\pi(\hat{\mathcal{N}}(v)_{t_B} \cup \{u_t\} ) - R_\pi(\hat{\mathcal{N}}(v)_{t_B} ) )
\end{align}
whenever $\hat{\mathcal{N}}(v)_{t_A}\subseteq\hat{\mathcal{N}}(v)_{t_B}$ and $t_A < t_B$
\end{prop}
\begin{proof}
We define,
\begin{align}
r_{t_A} &= R_\pi(\hat{\mathcal{N}}(v)_{t_A} \cup \{u_t\}) - R_\pi(\hat{\mathcal{N}}(v)_{t_A} )\\
 r_{t_B} &= R_\pi(\hat{\mathcal{N}}(v)_{t_B} \cup \{u_t\}) - R_\pi(\hat{\mathcal{N}}(v)_{t_B})
\end{align}

Then we need to show $ r_{t_A} \geq r_{t_B}$. Based on the reward definition in Eq.~(\ref{eq:reward}), we have,
\begin{align}
r_{t_A} - r_{t_B} = &
\frac{f_c(\textsc{Agg}(x_v,\{x_{u_t}\})}{\sum_{\tilde{u}\in\hat{\mathcal{N}}(v)_{t_A}}f_c(\textsc{Agg}(x_v,\{x_{\tilde{u}}\})}\nonumber\\
& - \frac{f_c(\textsc{Agg}(x_v,\{x_{u_t}\})}{\sum_{\bar{u}\in\hat{\mathcal{N}}(v)_{t_B}}f_c(\textsc{Agg}(x_v,\{x_{\bar{u}}\})}
\end{align}
Assume $F(t_A) = \sum_{\tilde{u}\in\hat{\mathcal{N}}(v)_{t_A}}f_c(\textsc{Agg}(x_v,\{x_{\tilde{u}}\})$, and $F(t_B) = \sum_{\bar{u}\in\hat{\mathcal{N}}(v)_{t_B}}f_c(\textsc{Agg}(x_v,\{x_{\bar{u}}\})$, the above equations can rewritten as,
\begin{align}
r_{t_A} - r_{t_B} =  \frac{f_c(\textsc{Agg}(x_v,\{x_{u_t}\})}{F(t_A) \cdot F(t_B)}(F(t_B) - F(t_A))
\end{align}
We have $F(t_B) - F(t_A)\geq 0$ based on the monotonicity property. Thus we have $ r_{t_A} \geq r_{t_B}$.
\end{proof}

\subsection{Equivalence between GDPNet and Submodular Maximization Problem}

We will establish the following facts, which together imply the equivalence between GDPNet and submodular maximization problem,
\begin{itemize}
	\item The total reward function $R_\pi$ defined in the signal neighbor selection phase is a submodular function, which is equivalent to $f_s: 2^N \rightarrow \mathbb{R}$. 
	\item The submodular maximization problem can be formulated as an MDP which is equivalent to GDPNet. 
	\item The objective function in GDPNet is equivalent to the counterpart in submodular maximization. 
\end{itemize}

Firstly, the goal of submodular maximization is to find the $w^*$, with the objective function:
\begin{equation}
\label{woptimize}
w^*(\Omega) = {\arg\max}_{A\subseteq  \Omega :\vert A\vert \leq K} f_s(A)
\end{equation}
where $f_s(A)$ is the cumulative value of each element in set $A$. Let $A$ be the selected neighborhood set $\hat{\mathcal{N}}(v)$, then $R_\pi(\hat{\mathcal{N}}(v)) = f_s(\hat{\mathcal{N}}(v))$.
 
Secondly, the submodular maximization problem can be formulated as an MDP where the set $A$ with the selected items indicates the state. After adding a new item $c$ into $A$, the state is updated to $A\cup\{c\}$. 

Lastly, the objective function of GDPNet also aligns to the optimizer $w$ in submodular maximization, where each $\pi_\theta$ can be considered as an optimizer $w$ in Equation~(\ref{woptimize}):
\begin{equation}
\pi_\theta^* = {\arg\max}_\theta \mathbb{E}_{A} R_\pi(A) 
\end{equation}
where $A$ is equivalent to the signal neighbor set $\hat{\mathcal{N}}(v)$

\begin{theorem}
Greedy gives a $(1-\frac{1}{e})$-approximation for the problem of ${\arg\max}_\theta \mathbb{E}_{A} R_\pi(A)$ when $R_\pi(A): 2^N \rightarrow \R$ is a monotone submodular function.
\end{theorem}
\begin{proof}
Based on the aforementioned equivalence between the objective functions of GDPNet and submodular maximization, we need to show that a $(1-\frac{1}{e})$-approximation solution can be achieved for $w^*(\Omega) = {\arg\max}_{A\subseteq  \Omega :\vert A\vert \leq K} f_s(A)$ with a submodular function $f_s$, which has been proved in~\cite{nemhauser1978analysis}.
\end{proof}

%% file: tex/results.tex
\section{Experiment}\label{sec:results}

Experiments are conducted to evaluate the robustness of the representations learned by the proposed GDPNet model. As for quantitative experiments, we focus on two tasks: (1) \emph{Robustness Evaluation}, we use micro-averaged F1 score to evaluate our model against baselines on node classification task, and (2) \emph{Denoising Evaluation}, we evaluate the denoising capability of GDPNet by comparing with baselines running on the denoised graph generated by GDPNet. We extract four datasets Cora, Citeseer, PubMed and DBLP followed by spliting them for training, test and validation with the supervised learning scenario which follows the previous work~\cite{hamilton2017inductive,velivckovic2017graph,chen2018fastgcn}. As for qualitative experiments, we conduct the embedding visualization which projects the learned high-dimension representations to a 2D space. In all these experiments, we separate out test data from training and perform predictions on nodes that are not seen during training.
\begin{table}\label{tab:data}
\centering
\caption{Basic statistics of datasets}
\begin{tabular}{cccccc} \toprule
Dataset & Nodes & Edges&Classes&Features&Train/Validate/Test\\\midrule
Cora&2,708&5,429&7&1,433&1,208/500/1,000\\
Citeseer& 4,230&5,358&6&602&2,730/500/1,000\\
PubMed&19,717 &44,338&3&500&18,217/500/1,000\\
DBLP&17,716 &105,734&4&1,639&16,216/500/1,000\\
\bottomrule
\end{tabular}
\end{table}

\subsection{Experimental Setup and Baselines}

For all these tasks, we apply a two-layer policy network to select the signal neighbors. 
The architectural hyper-parameters are optimized on the Cora dataset and shared by the other datasets. 
The embedding dimension is $128$. The size of the two hidden layers in policy network are $64$ and $36$, respectively, with active function ReLU. The batch size is $256$. The discount factor is optimized as $0.95$ for Cora and DBLP, $0.9$ for PubMed and $1.0$ for Citeseer. We compare our method with the following baselines: 
\begin{itemize}
    \item \textbf{LR}: Logistic regression (LR) model which takes the node features as inputs, and ignores graph structure.
    \item \textbf{GCN~\cite{kipf2016semi}}: GCN uses the local connection structure of the graph as the filter to perform convolution, where filter parameters are shared over all locations in the graph. We use inductive version of GCN in this paper for comparison
    \item \textbf{GAT~\cite{velivckovic2017graph}}: GAT utilizes the attention mechanism to enhance the performance of the graph convolutional network by considering the entire neighborhoods.
    \item \textbf{FastGCN~\cite{chen2018fastgcn}}: FastGCN considers graph convolutions as integral transforms of embedding functions, and samples the neighborhoods in each layer independently to addresses the recursive expansion of neighborhoods.
    \item  \textbf{GraphSAGE~\cite{hamilton2017inductive}} GraphSAGE extends the original graph convolution-style framework to the inductive setting. It randomly samples a fixed-size neighborhood of each node followed by performing a specific aggregator over it.
    
\end{itemize}

Our proposed model is denoted as \textbf{GDPNet}. We also introduce a variant \textbf{GDPNet$_\textit{RO}$} which performs the signal neighbor selection with a random order of the neighbors.

\begin{figure*}
	\centering
	\subfigure[Cora-GDPNet]{%
		\includegraphics[ width=1.64in]{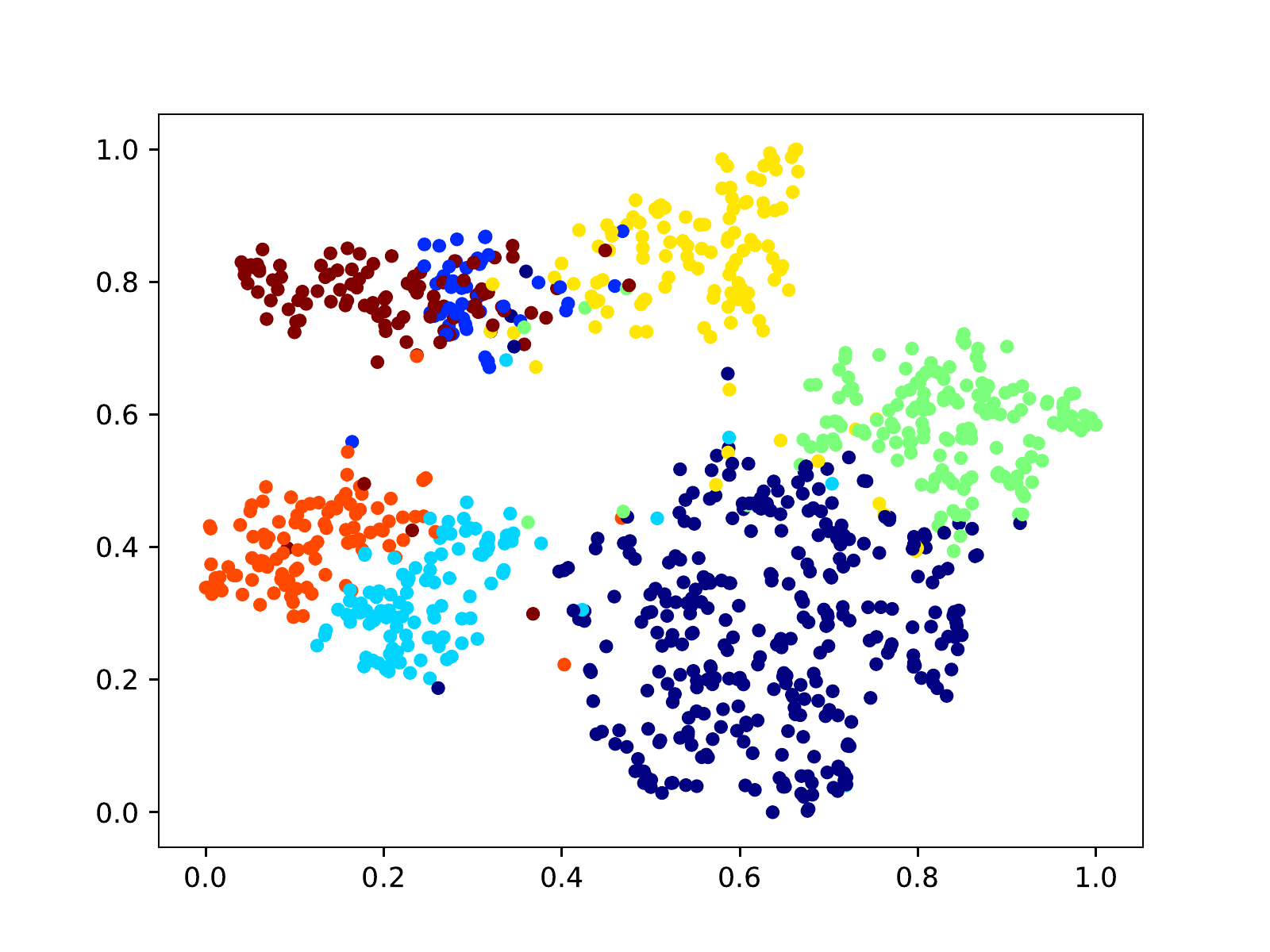}
		\label{fig:subfigure5}}
~
	\subfigure[Cora-GAT]{%
		\includegraphics[ width=1.64in]{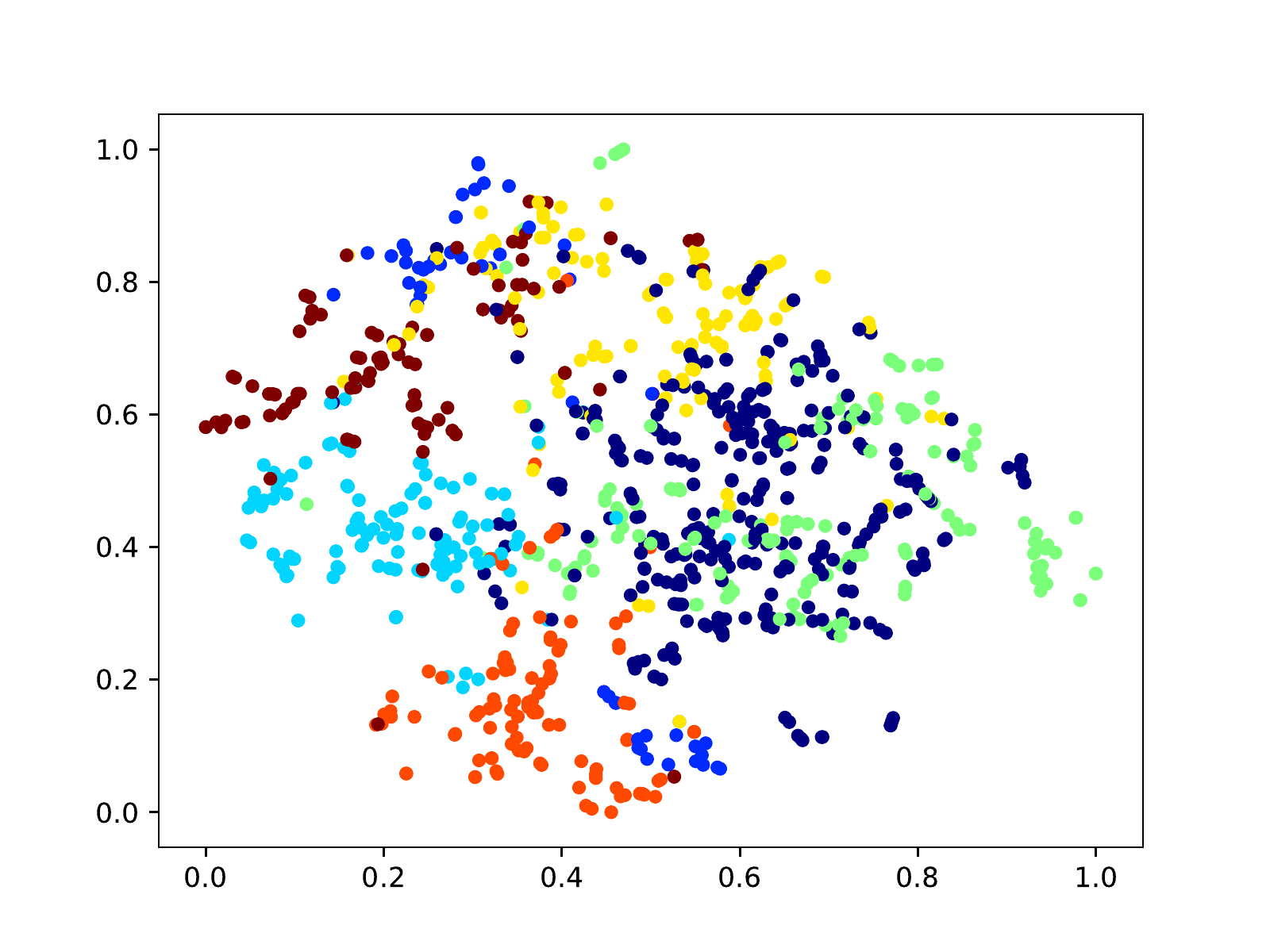}
		\label{fig:subfigure6}}
~
	\subfigure[Cora-GCN]{%
		\includegraphics[ width=1.64in]{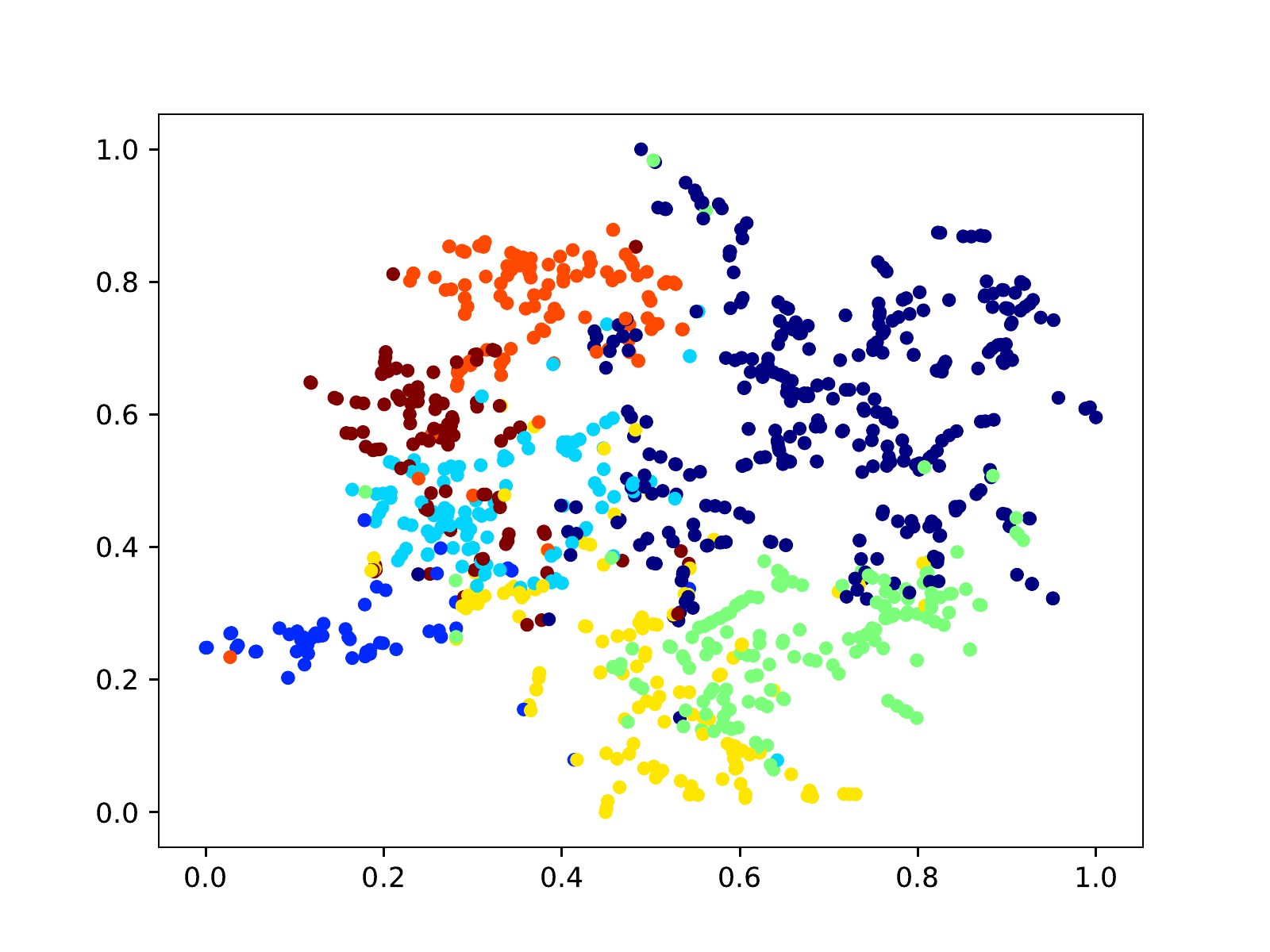}
		\label{fig:subfigure5}}
~
	\subfigure[Cora-GraphSAGE]{%
		\includegraphics[ width=1.64in]{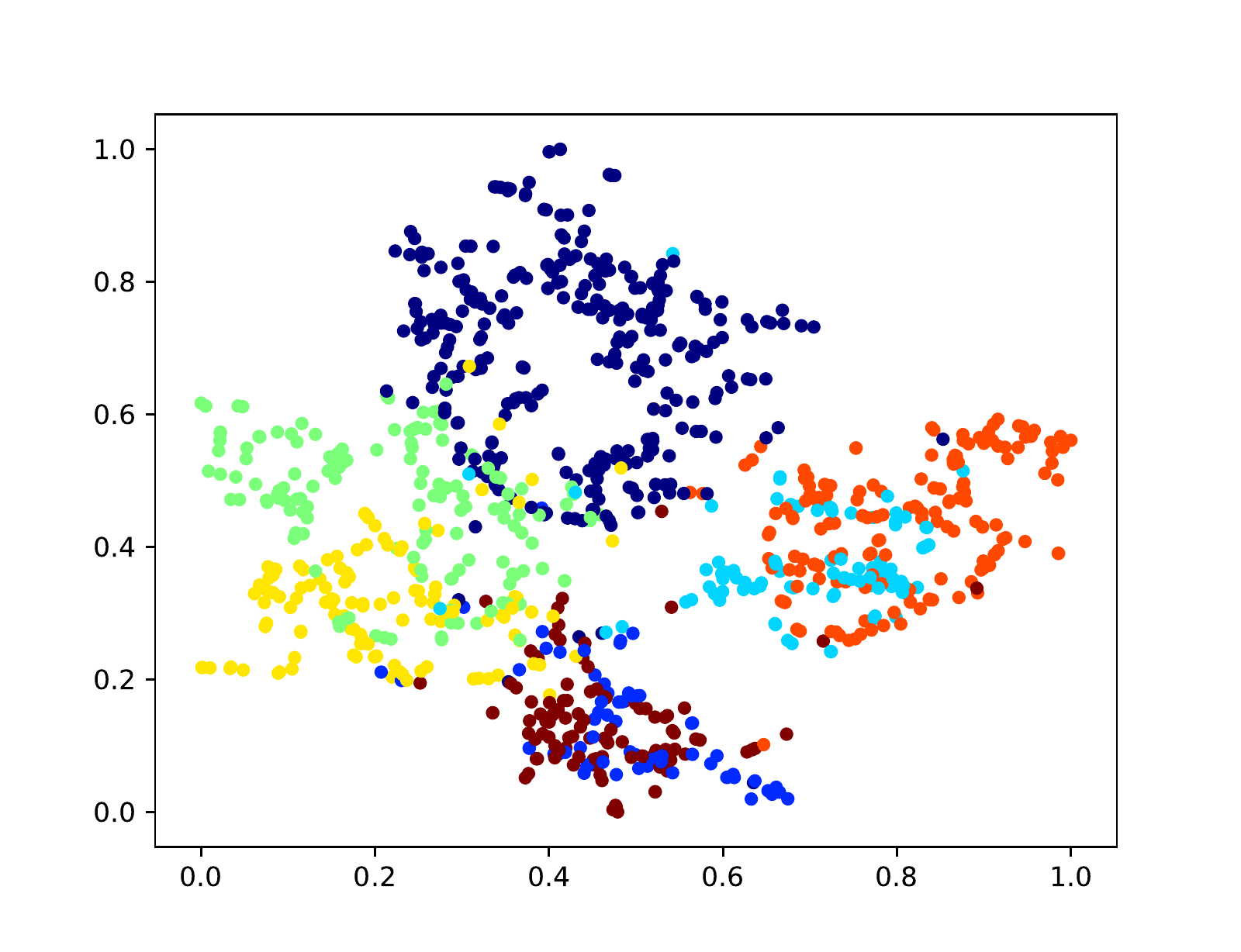}
		\label{fig:subfigure6}}

	\subfigure[Citeseer-GDPNet]{%
		\includegraphics[width=1.64in]{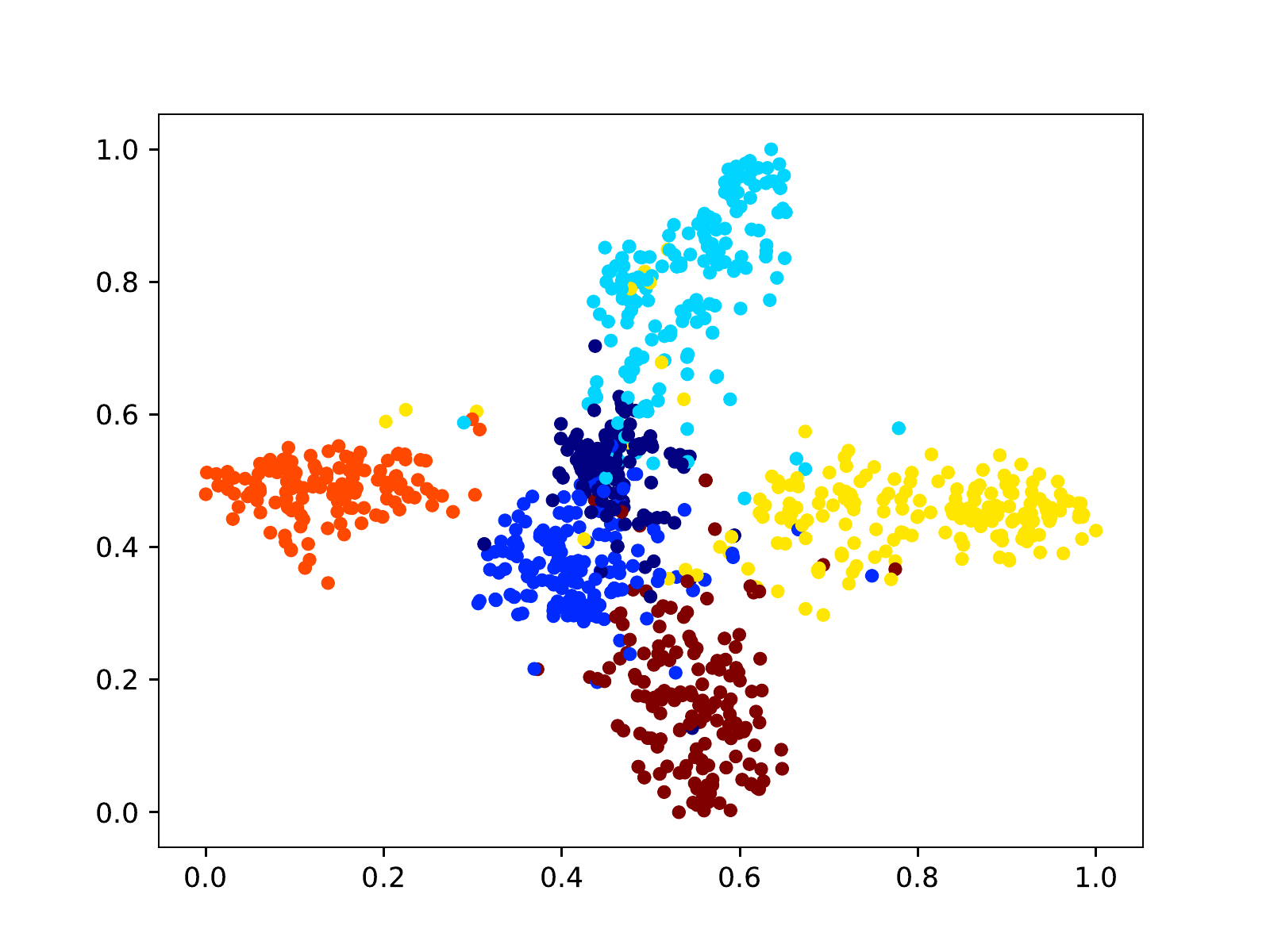}
		\label{fig:subf1}}
~
	\subfigure[Citeseer-GAT]{%
		\includegraphics[width=1.64in]{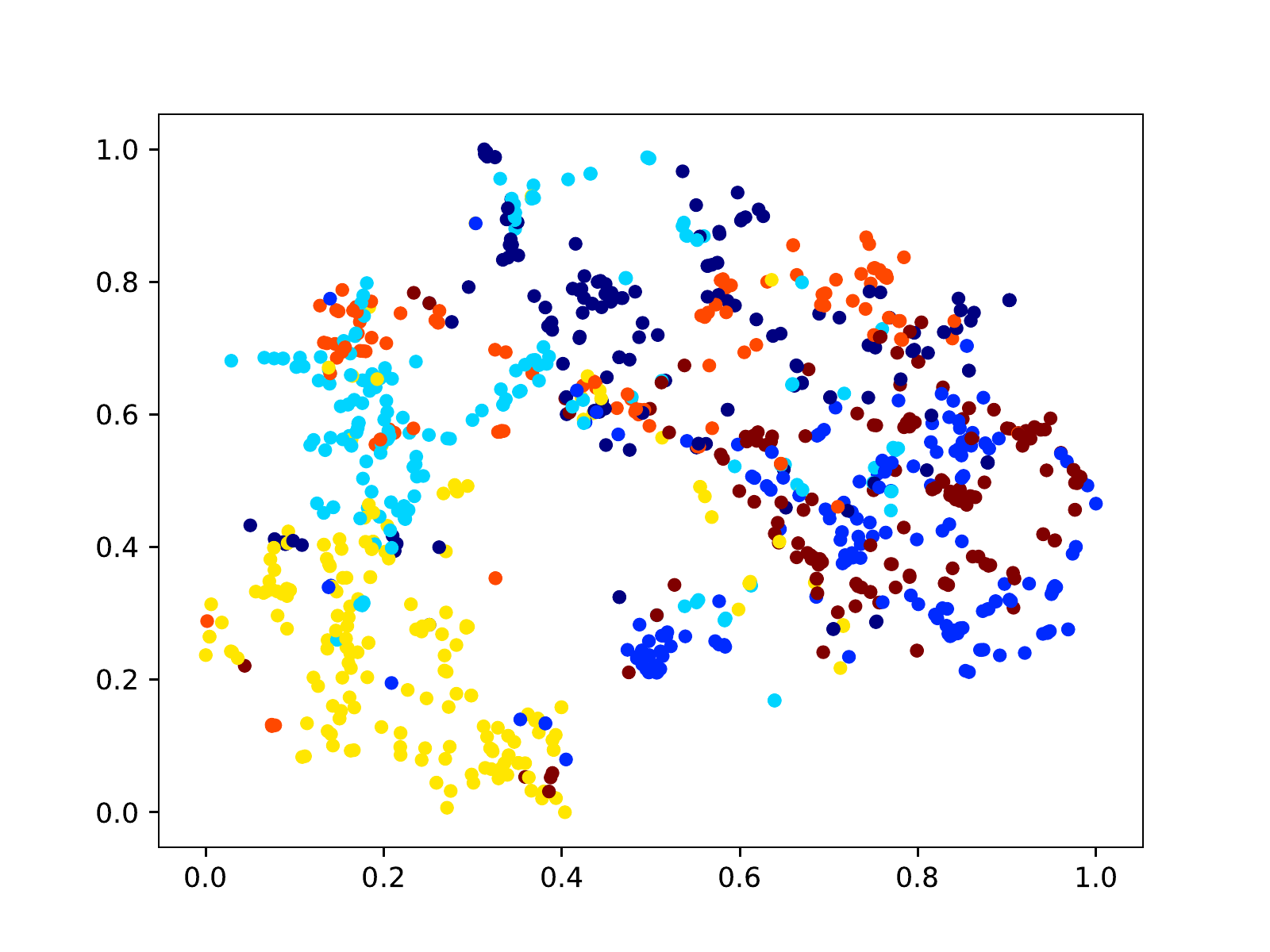}
		\label{fig:subfigure2}}
~
	\subfigure[Citeseer-GCN]{%
		\includegraphics[ width=1.64in]{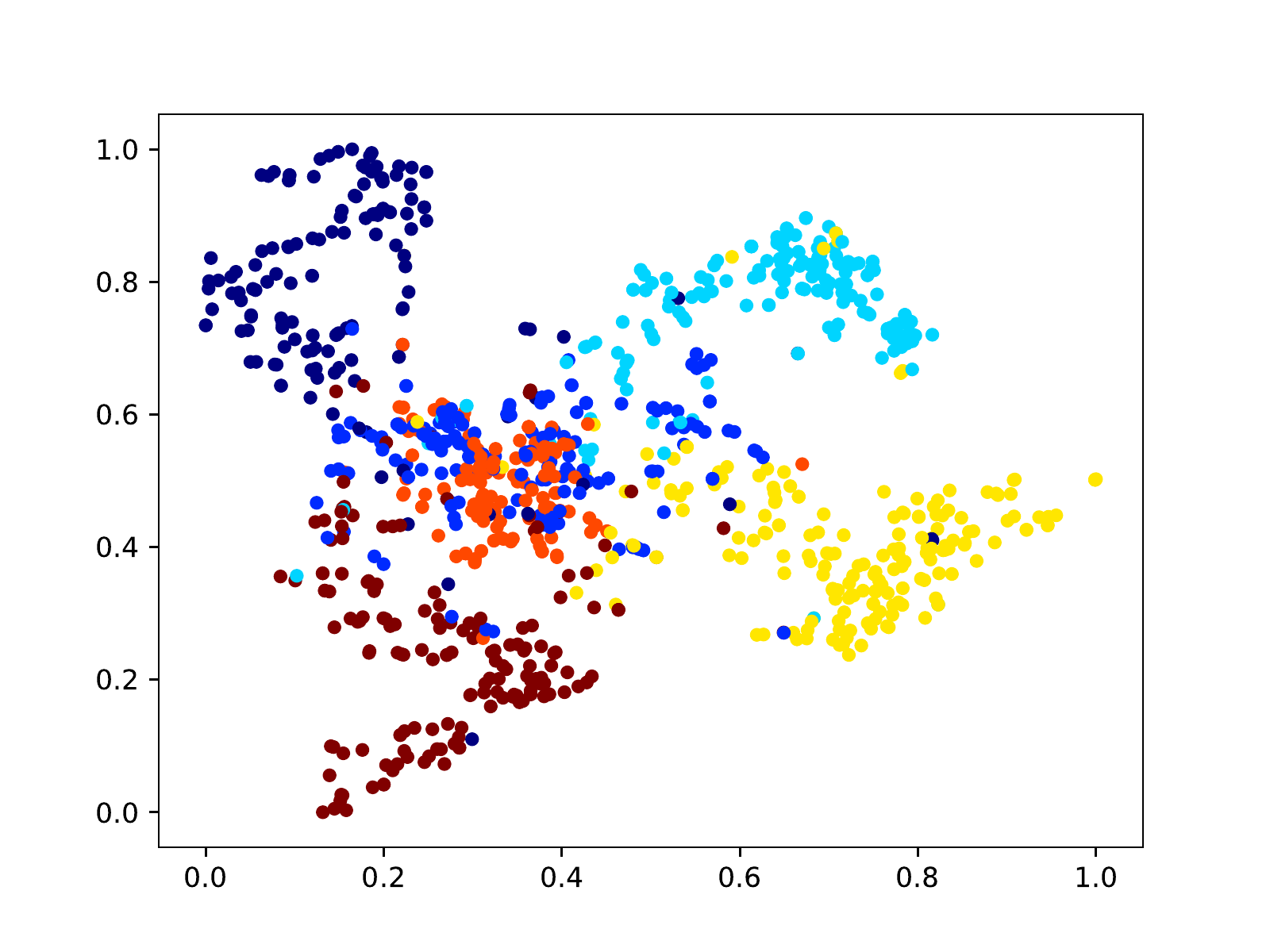}
		\label{fig:subfigure3}}
~
	\subfigure[Citeseer-GraphSAGE]{%
		\includegraphics[width=1.64in]{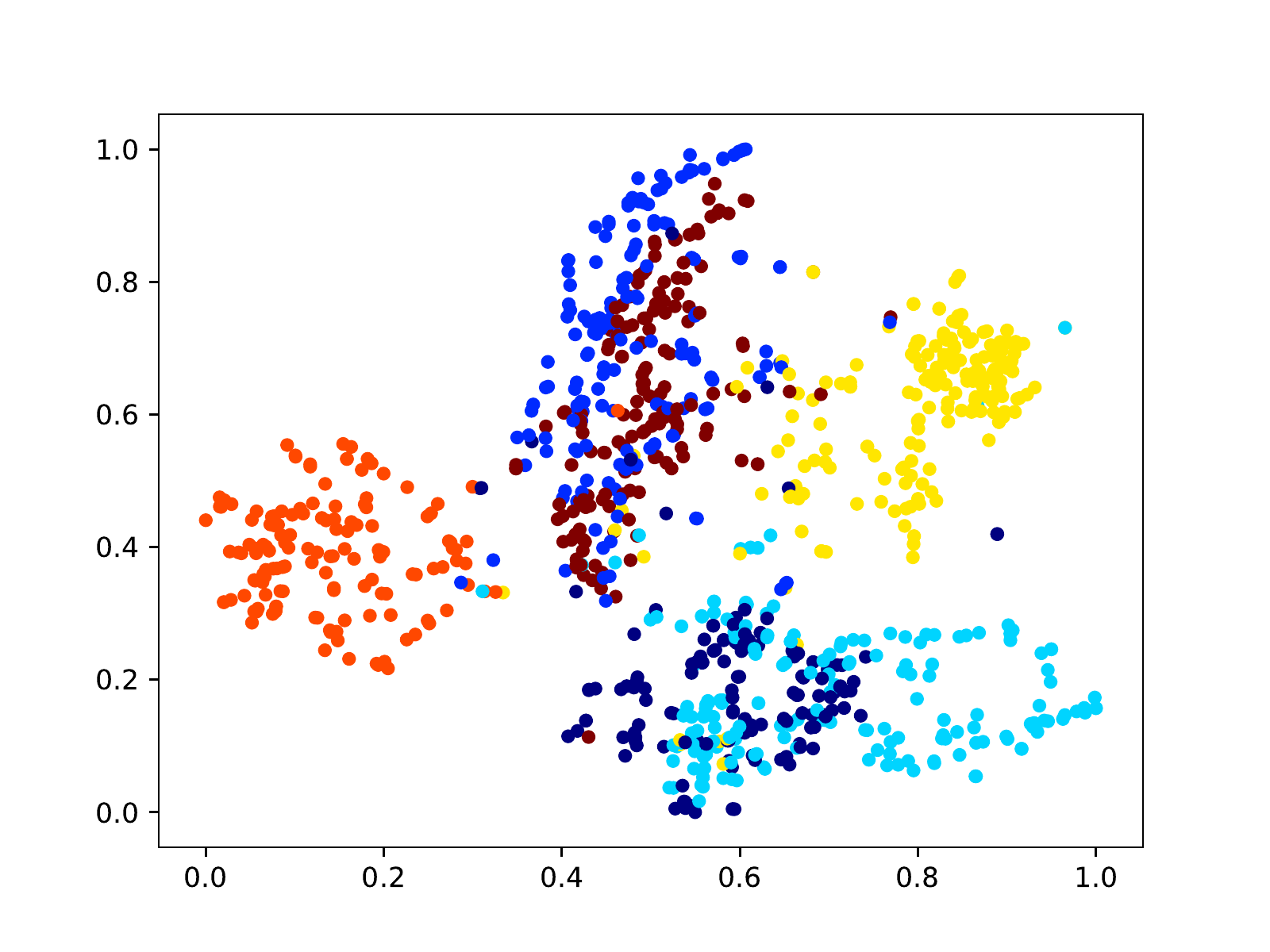}
		\label{fig:subfigure4}}

	\subfigure[PubMed-GDPNet]{%
		\includegraphics[ width=1.64in]{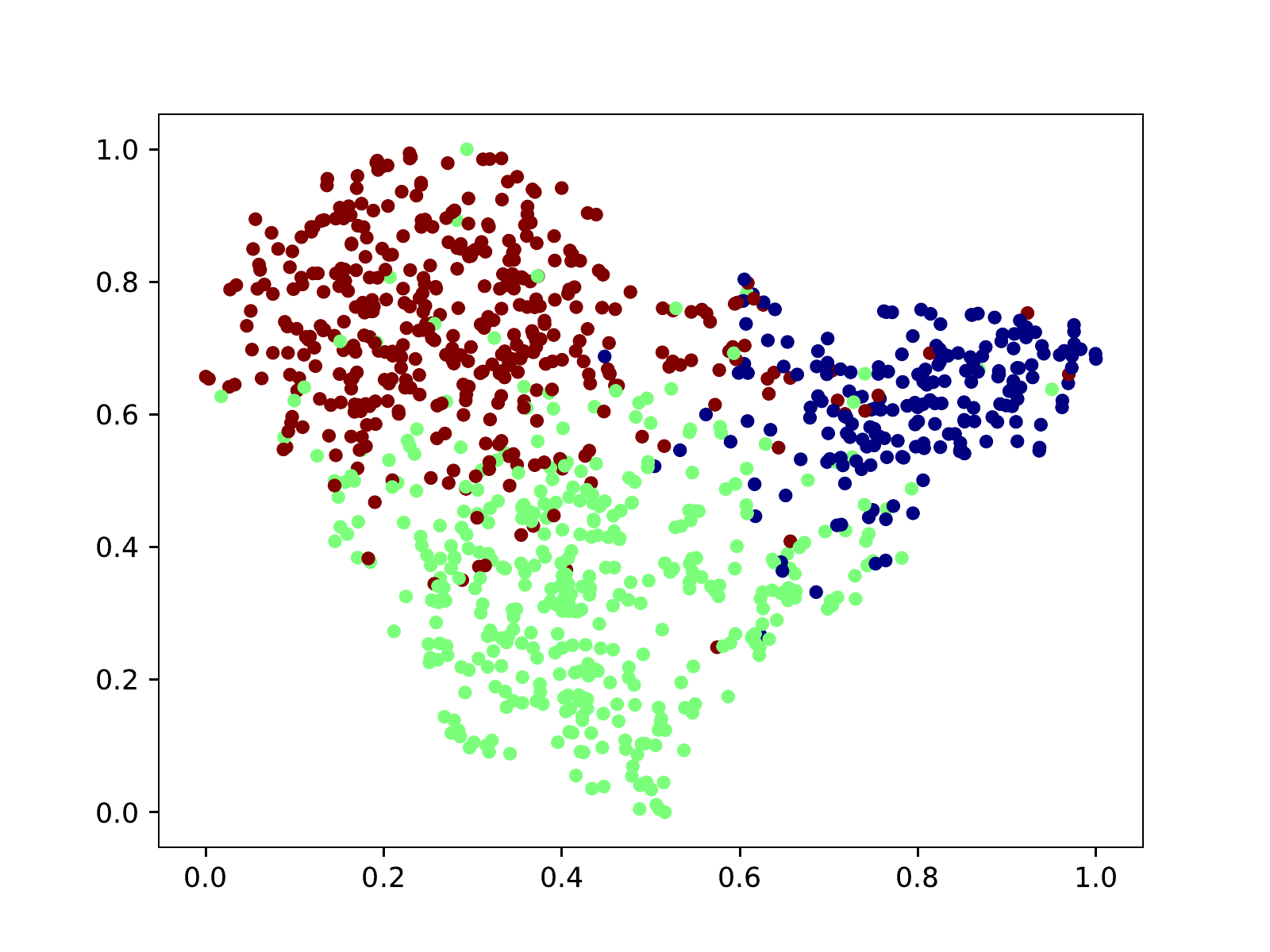}
		\label{fig:subfigure5}}
~
	\subfigure[PubMed-GAT]{%
		\includegraphics[ width=1.64in]{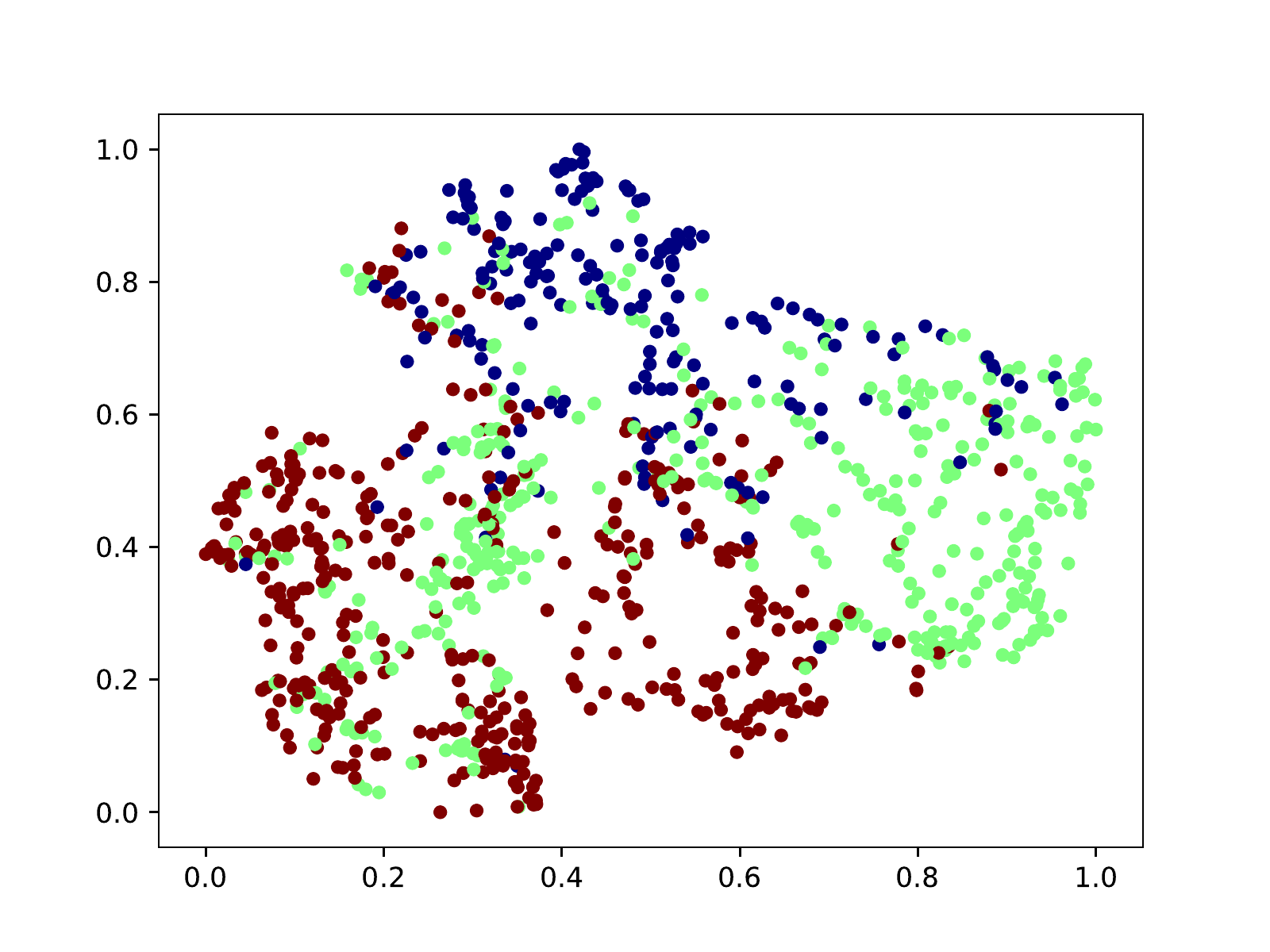}
		\label{fig:subfigure6}}
~
	\subfigure[PubMed-GCN]{%
		\includegraphics[ width=1.64in]{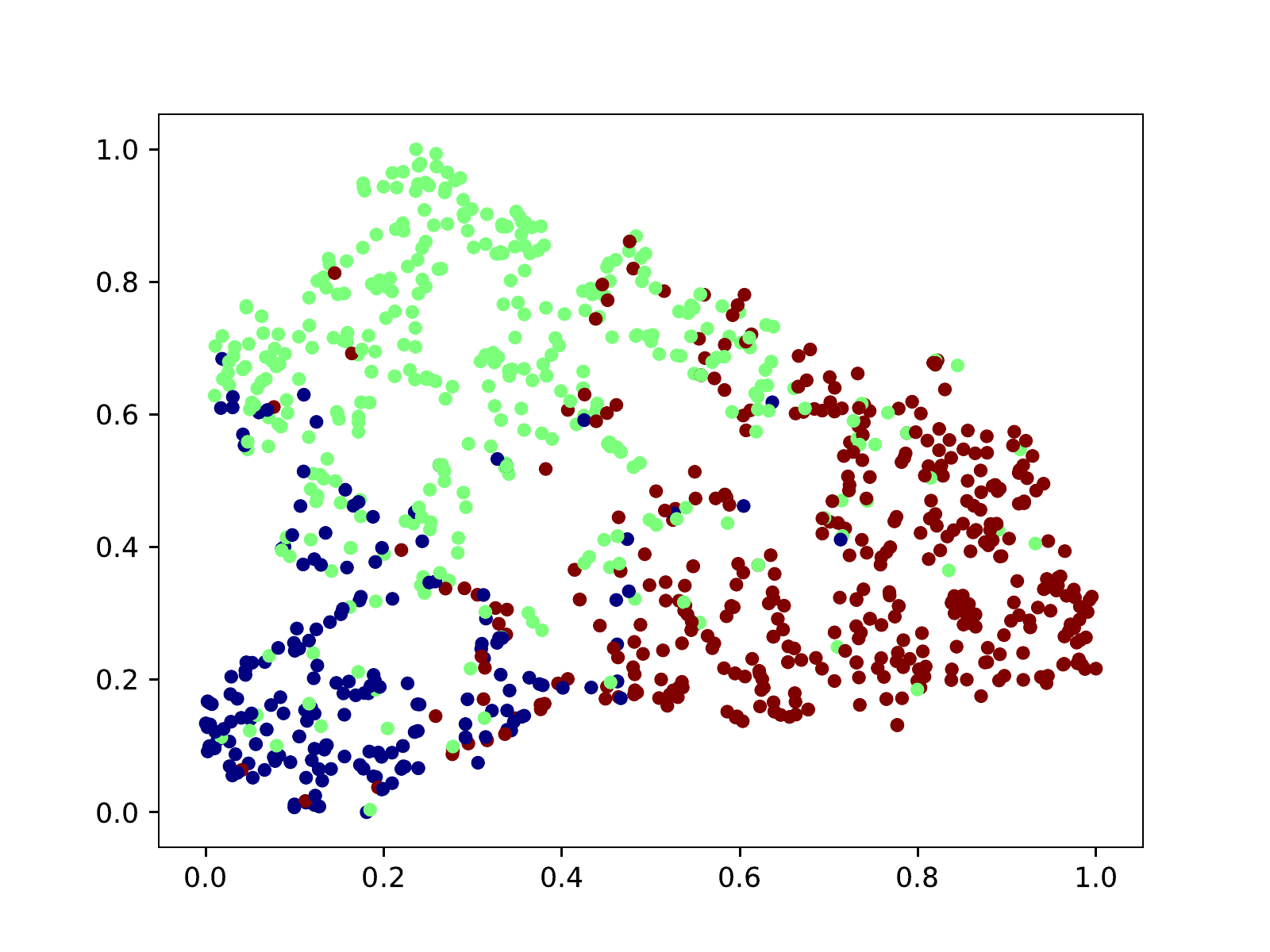}
		\label{fig:subfigure5}}
~
	\subfigure[PubMed-GraphSAGE]{%
		\includegraphics[ width=1.64in]{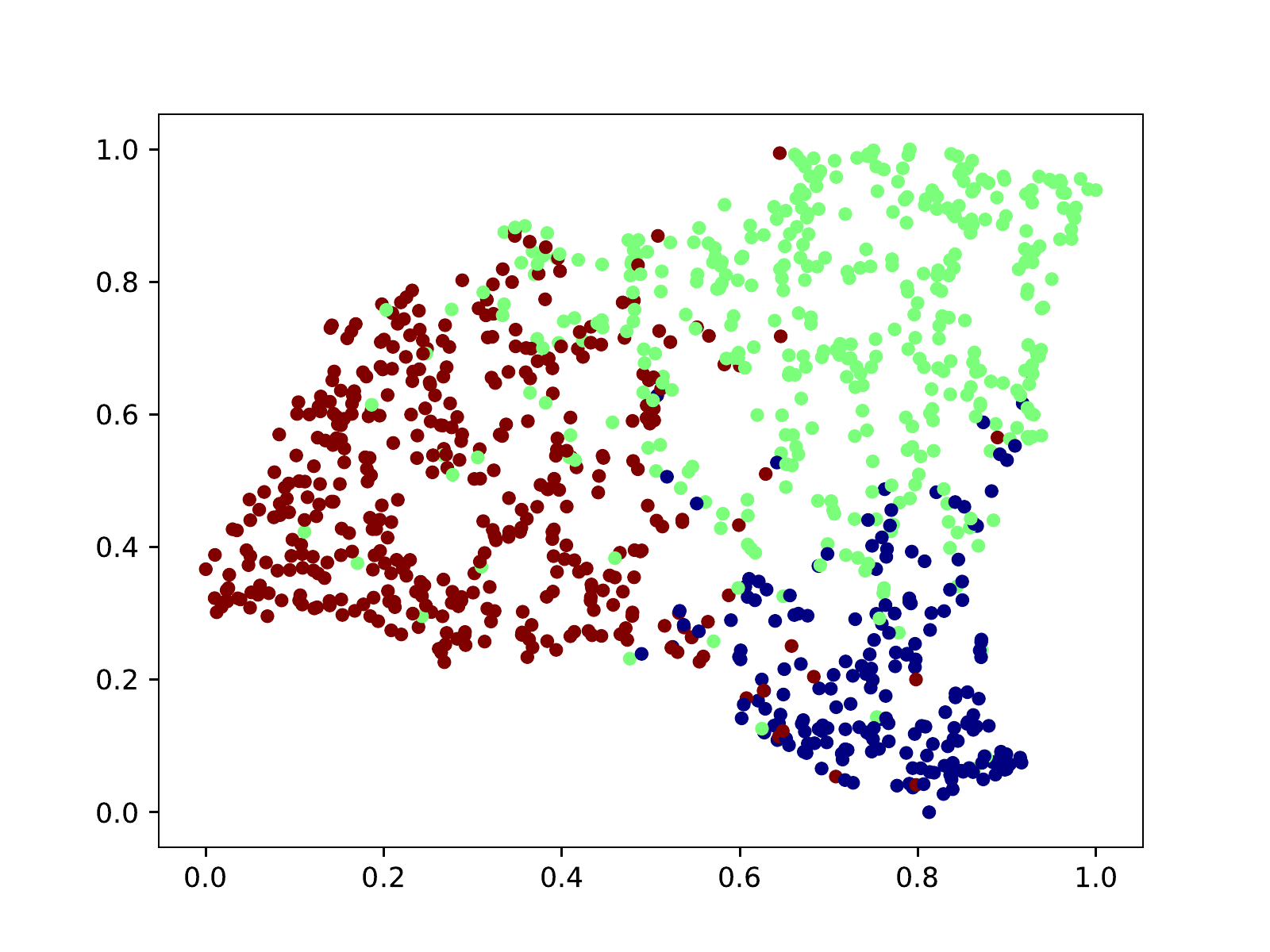}
		\label{fig:subfigure6}}

	\subfigure[DBLP-GDPNet]{%
		\includegraphics[ width=1.64in]{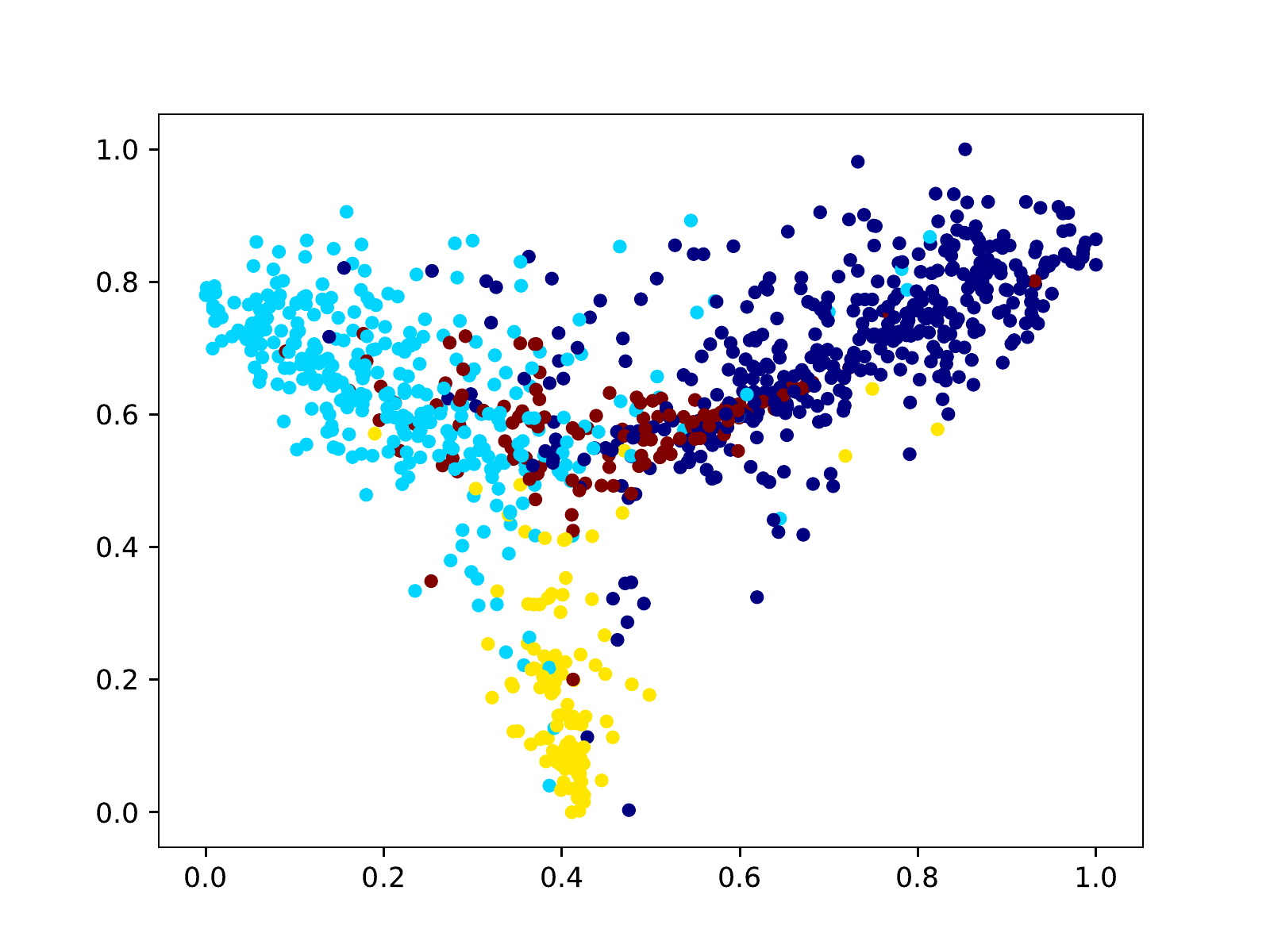}
		\label{fig:subfigure5}}
~
	\subfigure[DBLP-GAT]{%
		\includegraphics[ width=1.64in]{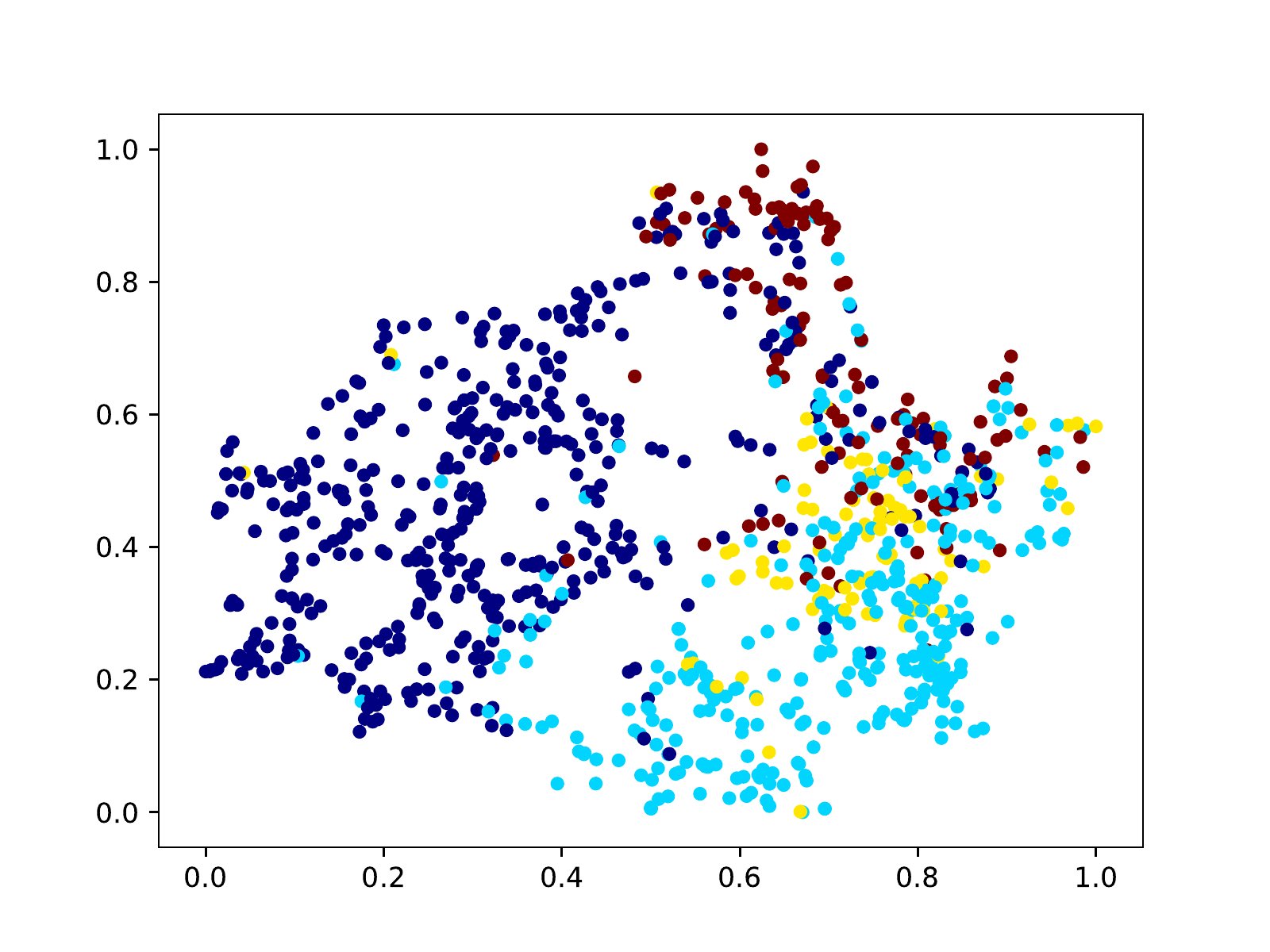}
		\label{fig:subfigure6}}
~
	\subfigure[DBLP-GCN]{%
		\includegraphics[ width=1.64in]{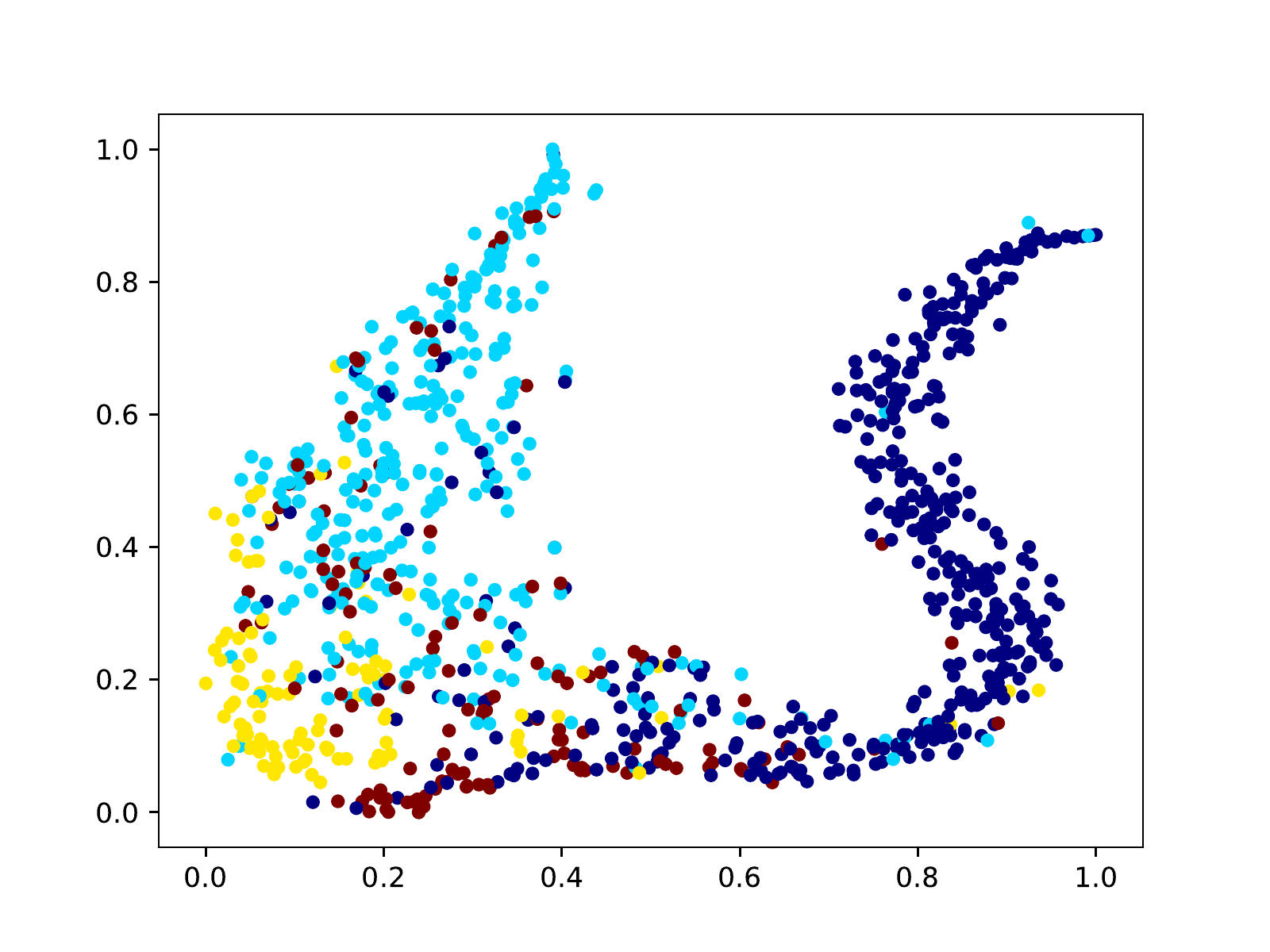}
		\label{fig:subfigure5}}
~
	\subfigure[DBLP-GraphSAGE]{%
		\includegraphics[ width=1.64in]{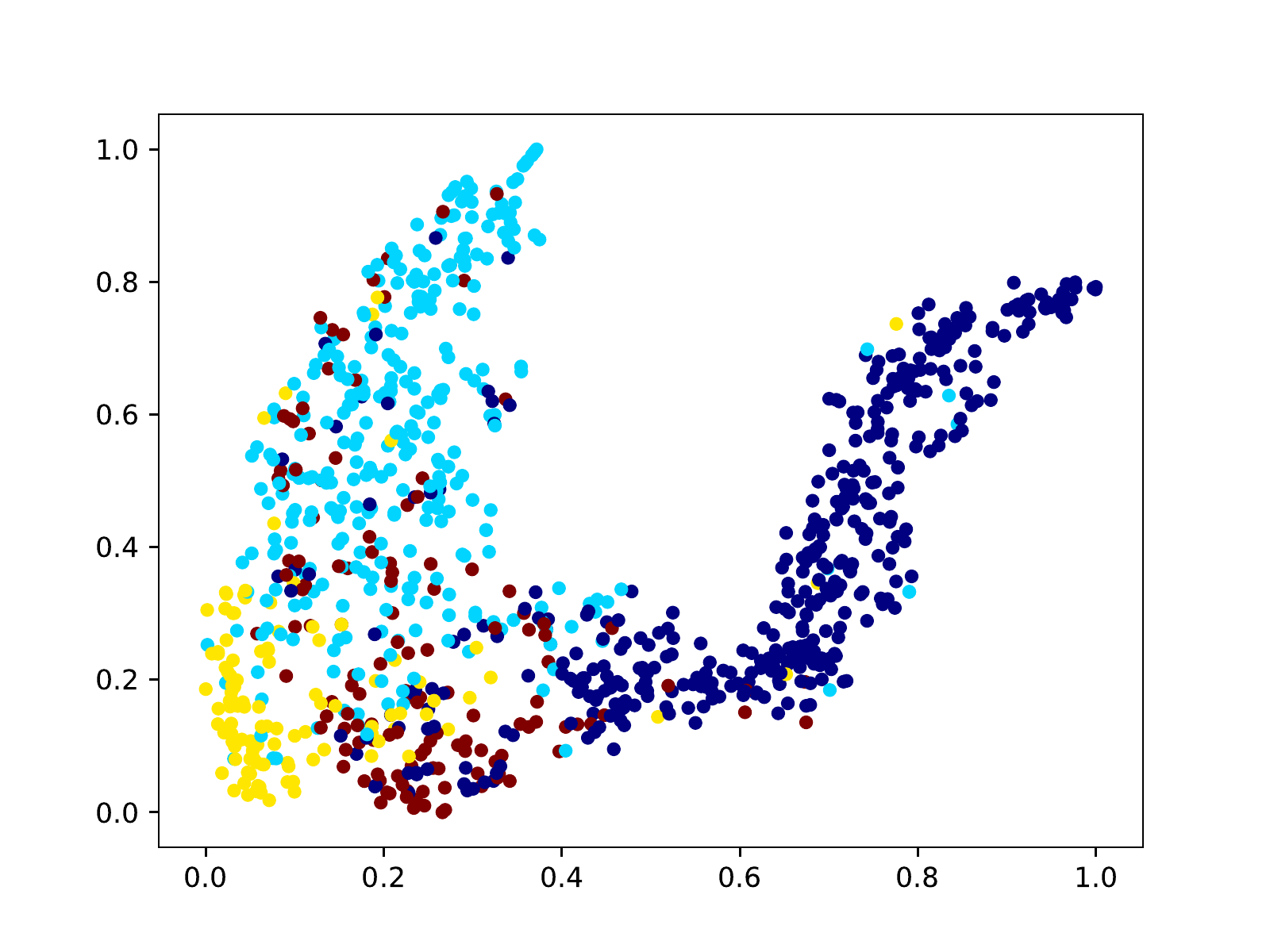}
		\label{fig:subfigure6}}
		
\caption{\label{fig:viz}Visualizations of the compared methods on Cora.}
\end{figure*}

\subsection{Performance Comparison}

In this section, we first visualize the node representations learned by different methods, followed by the performance comparison on node classification task. Additionally, we show the distributions of the selected signal neighbors with GDPNet on different dataset.

\subsubsection{Embedding Visualization}

Node representations are learned by GAT, GCN, GraphSAGE and GDPNet on test dataset of Cora, and visualized with t-SNE~\cite{maaten2008visualizing}, as shown in Fig.~\ref{fig:viz}. Different colors in the figure represent different categories in Cora. The following observations can be made from Fig.~\ref{fig:viz},
\begin{itemize}
	\item GDPNet correctly detects the classes in Cora, providing empirical evidence for the effectiveness of our method. This can be seen by the clear gap between samples with different colors. It also demonstrates that, removing the noisy neighbors can help nodes learn better representations.
	\item GCN and GraphSAGE share similar ``shape'' in the 2D space. The reason is that in the inductive learning setting, GCN and GraphSAGE use the same methods in neighborhood sampling. GAT considers the entire neighborhoods which leads to a different visualization result from the others. It can be seen that the sampled neighbors have a profound effect on the representations.
	\item GAT cannot effectively identify different classes as other methods, it might because it considers all the neighbors with attention weights, which is easily to introduce noisy neighbors.
\end{itemize}

\begin{table*}
\centering
\caption{\label{tab:classification}Summary of node classification results in terms of micro-averaged F1 score, for Cora, Citeseer, PubMed and DBLP}
\begin{tabular}{ccccc} \toprule
Method & Cora & PubMed&DBLP&Citeseer\\\midrule
LR	&$0.799 \pm 1.06\%$ &$0.871 \pm 0.82\%$&$0.784 \pm 1.03\%$&$0.813 \pm 0.58\%$\\
GAT	&$0.819\pm0.45\% $&$0.778\pm0.71\% $&$0.736\pm0.82\% $&$0.719\pm0.50\% $\\
GCN	&$0.838 \pm 0.50$\%&$0.826 \pm 0.22\%$&$0.805 \pm 2.17\%$&$0.829 \pm 1.56\%$\\
FastGCN	&$0.865\pm4.50\% $&$0.867\pm1.05\% $&$0.774\pm0.41\% $&$0.779\pm0.53\% $\\
GraphSAGE	&$0.867 \pm 0.52\%$& $0.854 \pm 0.87\%$ &$0.803 \pm 1.28\%$&$0.910\pm 0.73\%$\\\midrule
GDPNet$_\textit{RO}$& $0.879\pm 2.14\%$&$0.880 \pm 2.51\%$&$0.832 \pm 0.97\%$&$0.952 \pm 1.15\%$\\
GDPNet&$\mathbf{0.881\pm 0.31}\%$&$\mathbf{0.893 \pm 0.57\%}$&$\mathbf{0.836 \pm 0.57\%}$&$\mathbf{0.957 \pm 0.33\%}$\\
\bottomrule
\end{tabular}
\end{table*}


\begin{table*}
\centering
\caption{\label{tab:robust} Node classification results on original graph and the denoised graph by GDPNet, measured with micro-averaged F1 score}
\begin{tabular}{ccccc} \toprule
Method & Cora & PubMed&DBLP&Citeseer\\\midrule
GCN	&$0.838 \pm 0.50$\%&$0.826 \pm 0.22\%$&$0.805 \pm 2.17\%$&$0.829 \pm 1.56\%$\\
GCN$_\textit{GDPNet}$&$0.844\pm0.42\%$&$0.838\pm0.37\%$&$0.811 \pm 1.20\%$&$0.834 \pm 2.06\%$\\\midrule
GraphSAGE&$0.867 \pm 0.52\%$&$0.854 \pm 0.87\%$&$0.803 \pm 1.28\%$&$0.910\pm 0.73\%$\\
GraphSAGE$_\textit{GDPNet}$&$0.869\pm0.67\%$&$0.865\pm0.57\%$&$0.816 \pm 0.81\%$&$0.937 \pm 0.70\%$\\\bottomrule
\end{tabular}
\end{table*}

\begin{figure}
	\centering
	\includegraphics[width=8.5cm]{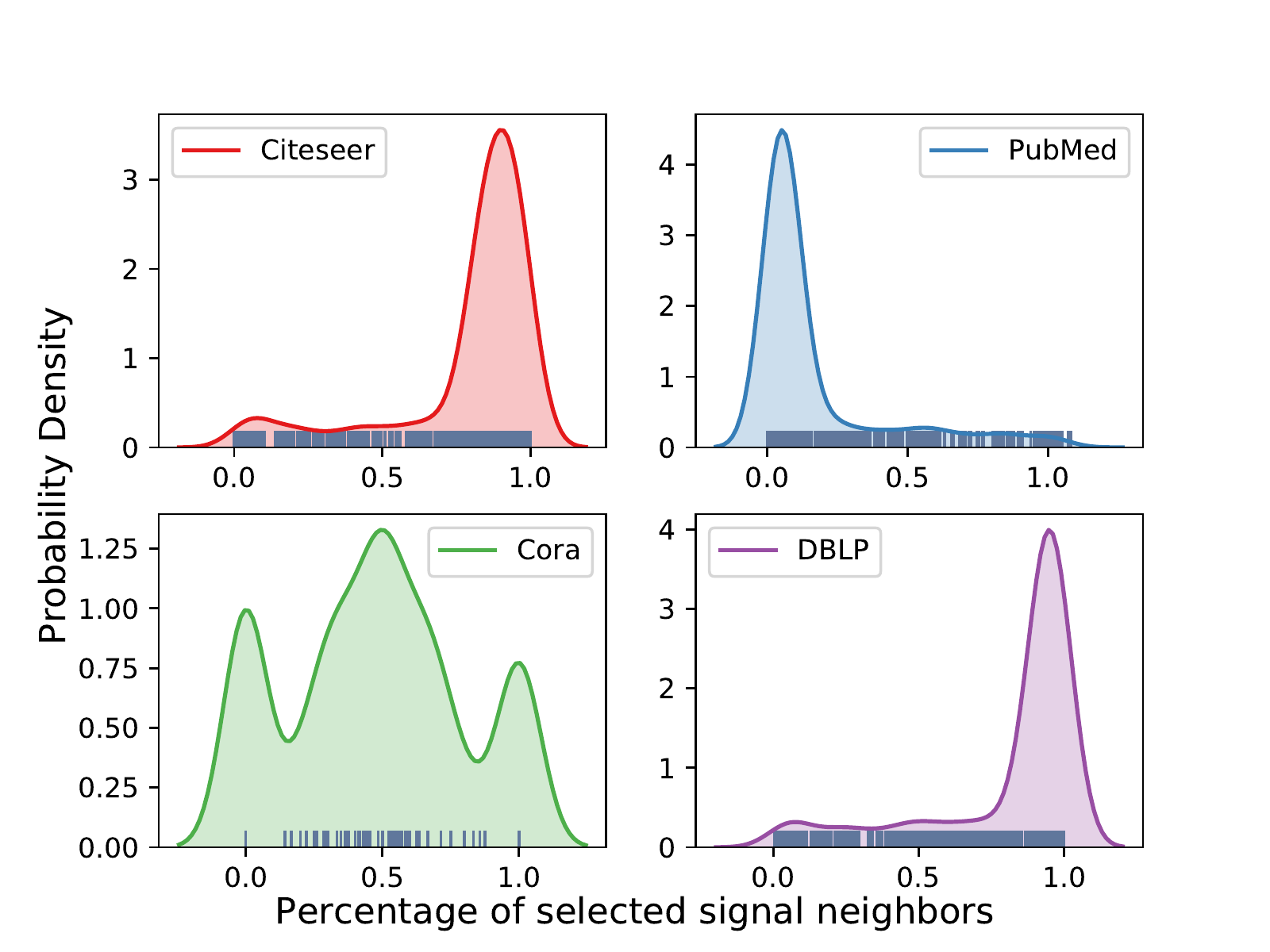}
	\caption{The distribution of the selected signal neighbor percentages.}
	\label{fig:distribution}
\end{figure}
\vspace{-0.5em}



\subsubsection{Results on Node Classification}

In this part, we compare the performance of GDPNet against the baselines on Cora, Citeseer, PubMed and DBLP. For all methods, we run the experiments with random seeds over $15$ trials and record the mean and standard variance of the \emph{micro-average F1} scores. The results are summarized in Table~\ref{tab:classification}. From the table we observe that,
\begin{itemize}
\item GDPNet consistently outperforms the other methods, which demonstrates there exists a set of noisy neighbors in each dataset on node classification task, and GDPNet can learn robust embeddings by effectively removing these noisy neighbors.
\item GCN, FastGCN and GraphSAGE show lower F1 scores. The reason is that these methods randomly sample a subset of neighbors for representation learning, which is hard to avoid the noisy neighbors. In addition, variance is higher via random sampling.
\item GAT learns the importance of the neighbors with attention weights, which is also sensitive to noisy data according to the reported results. 
\item Another interesting observation is that Logistic regression achieves better performance than the other baselines on PubMed, which indicates that there would be less signal neighbors for the nodes in PubMed. This observation can also be verified in Fig.\ref{fig:distribution}.
\item GDPNet$_\textit{RO}$ has a lower F1 score with higher variance than GDPNet, which demonstrates that the order of the decisions has an effect on the performance of representation learning. Thus learning an order for the neighbors is beneficial for selecting signal neighbors and robust graph representation learning.
\end{itemize}

\begin{figure}
\centering
\includegraphics[width=.495\linewidth]{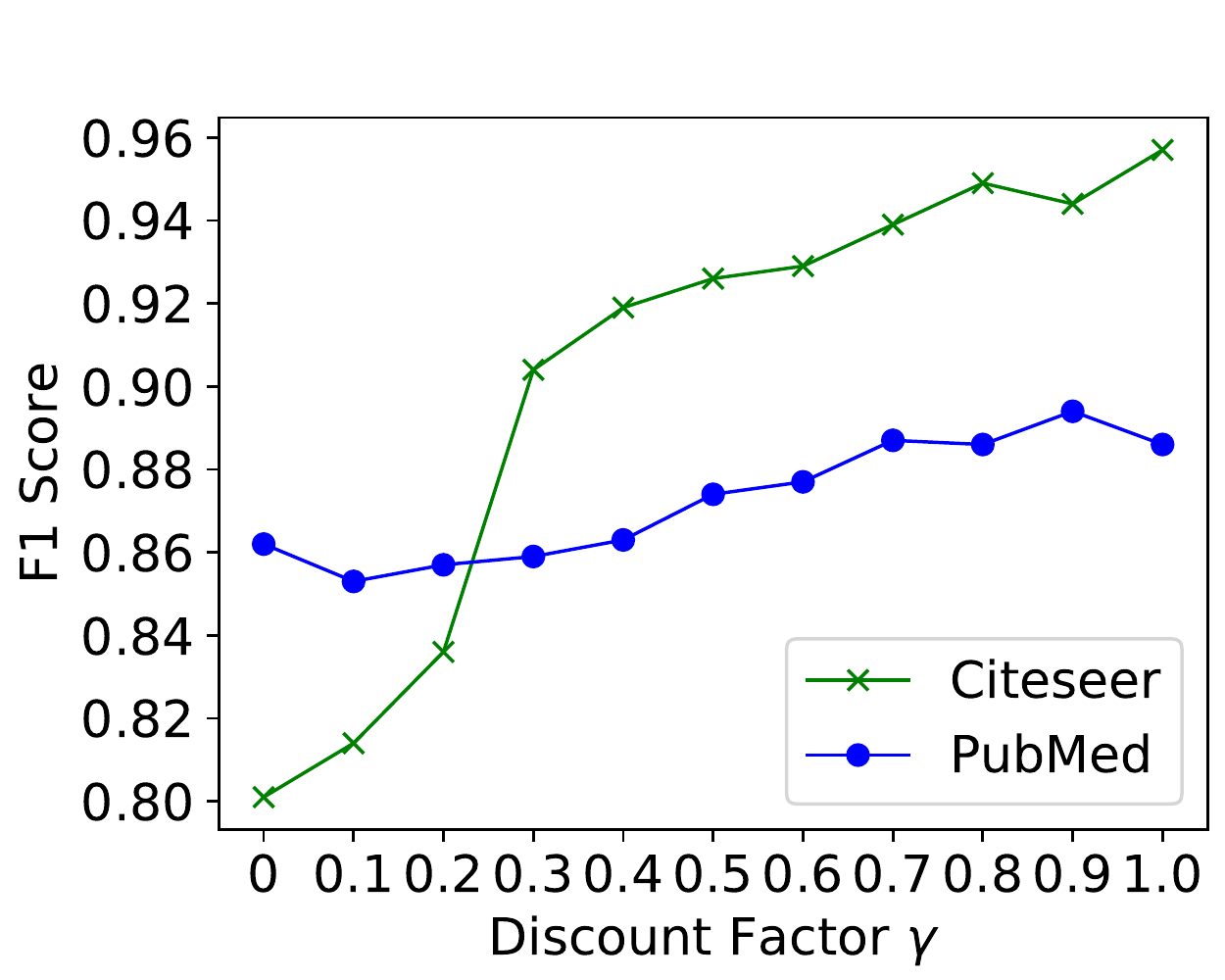}~
\includegraphics[width=.495\linewidth]{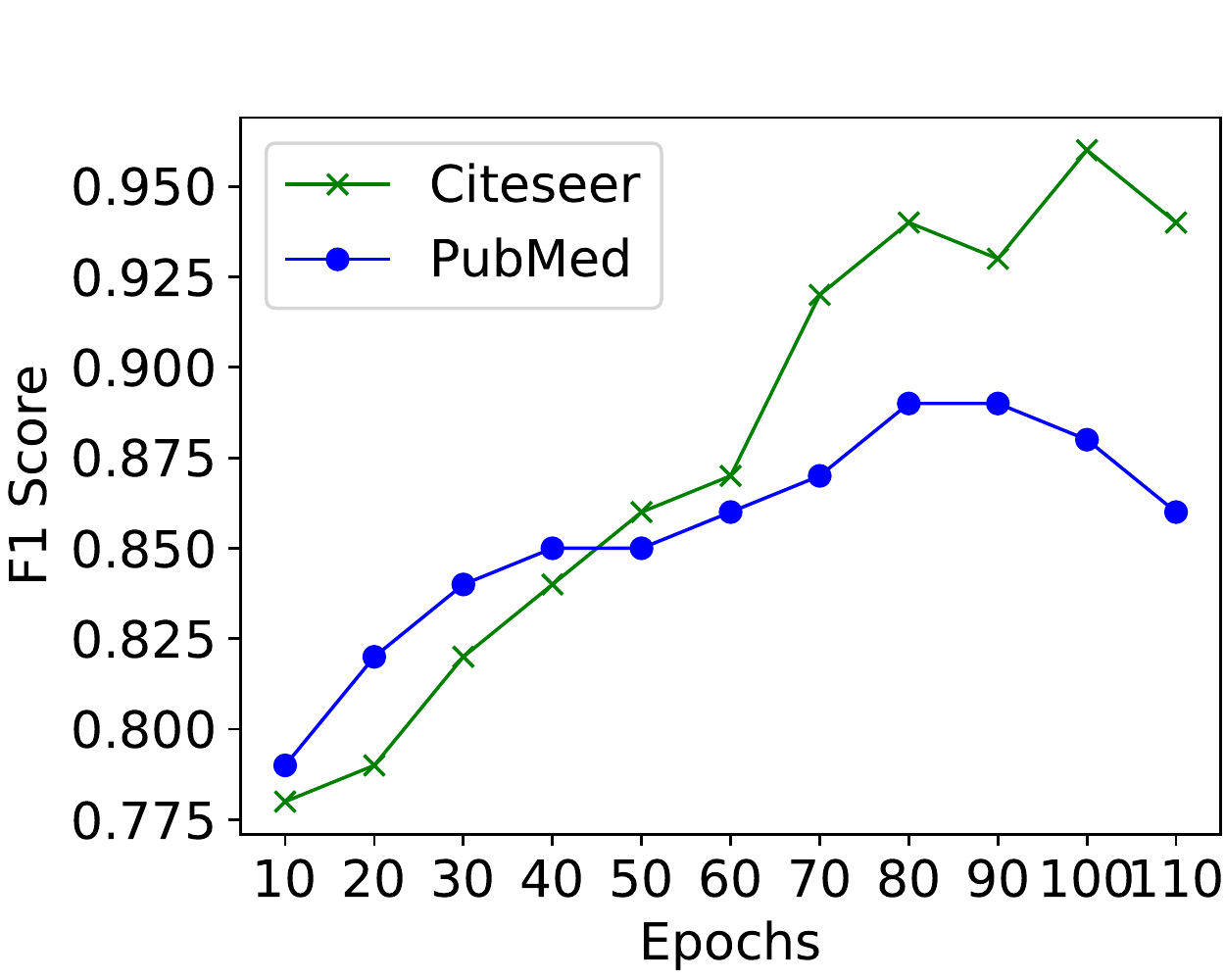}
\caption{Parameters analysis}
\label{fig:parameters}
\end{figure}

\begin{figure*}
  \centering
  \begin{minipage}[b]{0.66\textwidth}
\includegraphics[width=6.1cm]{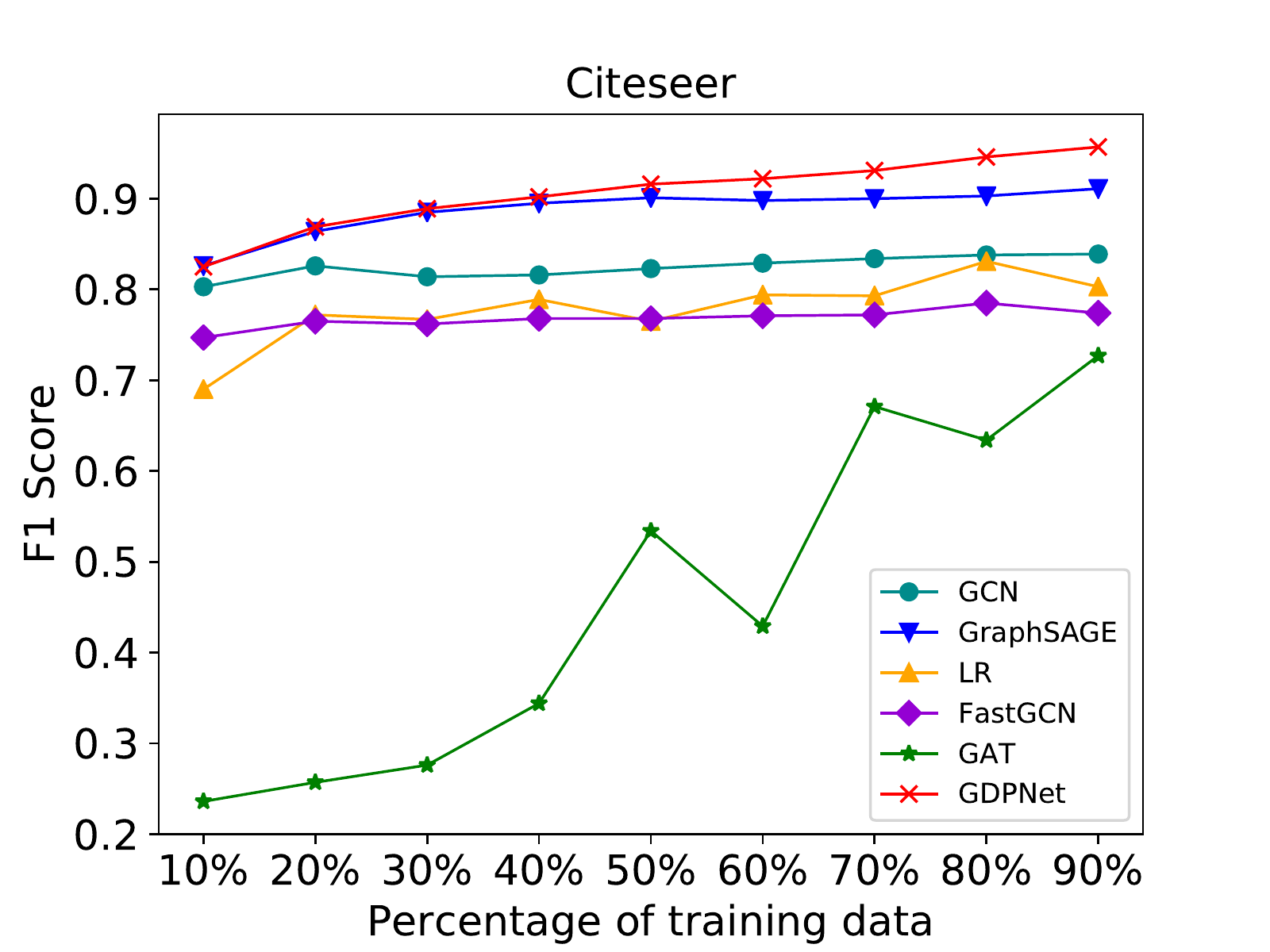}
\includegraphics[width=6.1cm]{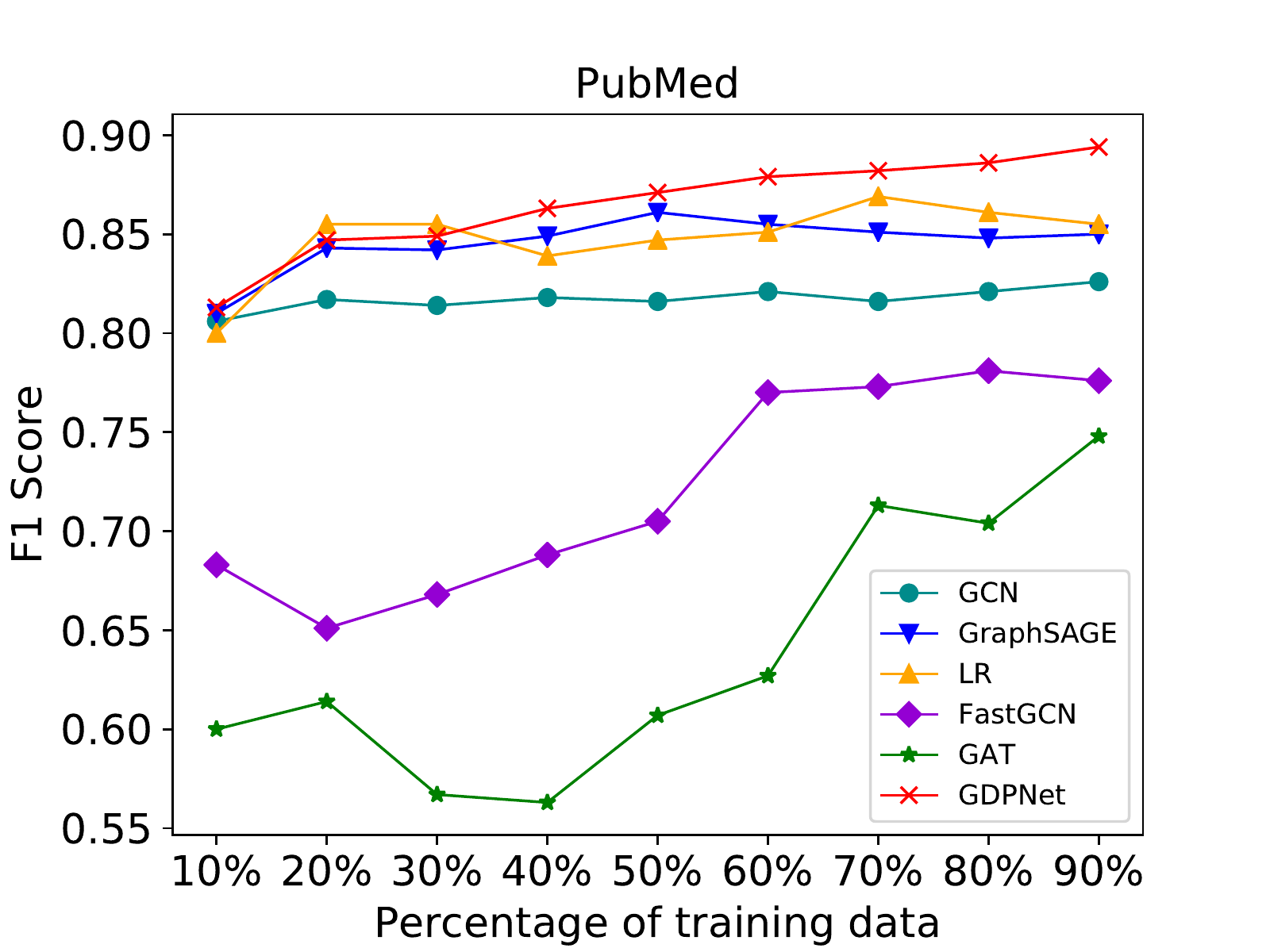}
\caption{Performance on different percentage of training data}
    \label{fig:percent}
  \end{minipage}
  \hfill
  \begin{minipage}[b]{0.33\textwidth}
\includegraphics[width=6.1cm]{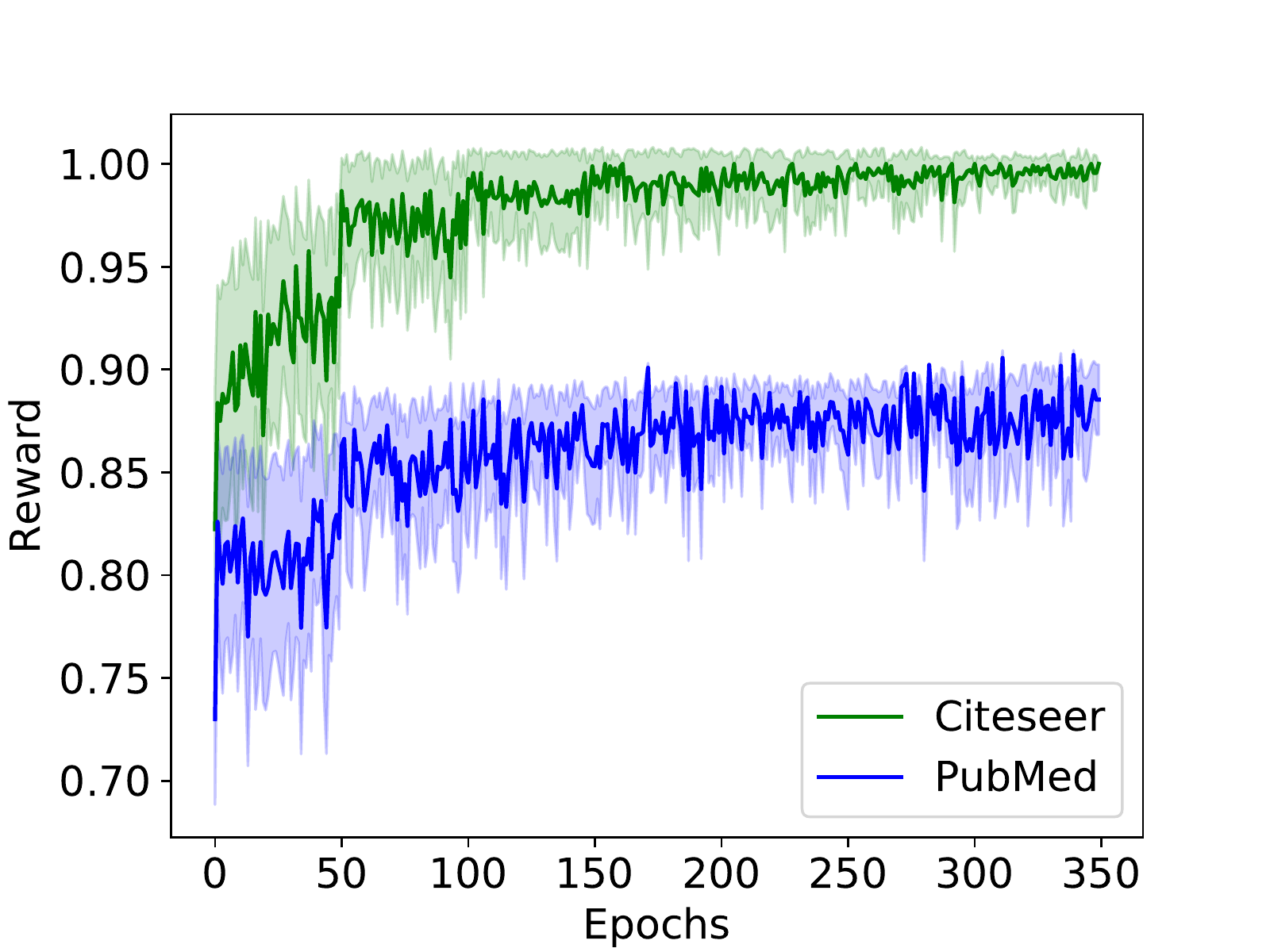}
    \caption{Convergence analysis}
    \label{fig:converge}
  \end{minipage}
\end{figure*}

\subsubsection{Distribution of the Selected Neighbors}

Fig.~\ref{fig:distribution} shows the distribution of the selected neighbor percentages, where the $x$-axis indicates the percentage of the nodes been selected as signal neighbors, and the $y$-axis indicates the probability densities. We observe that most of the neighbors in Citeseer and DBLP are selected while only a few neighbors are selected in PubMed. The results show that there would be more ``noisy'' citations (e.g. cross-field citation) in PubMed than in Citeseer and DBLP. Interestingly, most of the research papers collected in Citeseer and DBLP are from computer science, while PubMed collects papers from biomedical. 



\subsection{Ablation Study}
\subsubsection{Node classification performance comparison on selected signal neighbors}

In this part, we evaluate the effectiveness of denoising process in GDPNet. Specifically, we first utilize the policy learned by GDPNet to remove the noisy neighbors from Citeseer and PubMed. With the denoised graphs, we learn representations with GCN and GraphSAGE to see whether the performance can be improved on the denoised graphs. The results are summarized in Table~\ref{tab:robust}, where the suffix ``$_\textit{GDPNet}$'' indicates the results on the denoised graphs generated by GDPNet. As expected, both GCN and GraphSAGE achieves better performance on the denoised graphs, which demonstrates the effectiveness of the denoising process in GDPNet. 

\subsubsection{Parameter Sensitivity Study}

In Fig.~\ref{fig:percent}, we vary the training percentage of nodes in Citeseer and PubMed to test the classification accuracy. We observe that, the performance of all the methods are improved with the increases of the training percentage. Additionally, it can be seen that GAT is very sensitive to the percentages of training data, and it requires larger proportion of training data in order to have a desirable performance. GraphSAGE, GCN and GDPNet achieve good performances on small training data, and GDPNet make more improvements as the training data percentage increases.


Discount factor $\gamma$ balances the importance between instant reward and long-term reward. The large $\gamma$ indicates the more important role of long-term reward. Fig. 5 shows that when $\gamma = 1.0$, Citeseer achieves the best performance, while PubMed achieves best performance when $\gamma = 0.9$. We can see that Citeseer is more sensitive to the discount factor than PubMed.

Fig.~\ref{fig:parameters} presents the analysis on the number of epochs for representation learning phase. It can be seen from the figure that, with the increase of epochs (between $10$ and $80$), the performances of PubMed and Citeseer are both improved. The epochs to achieve best performance are $80$ and $100$ for PubMed and Citeseer, respectively.

In Fig.~\ref{fig:percent}, we vary the training percentage of nodes in Citeseer and PubMed to test the classification accuracy. We observe that, the performance of all the methods are improved with the increases of the training percentage. Additionally, it can be seen that GAT is very sensitive to the percentages of training data, and it requires larger proportion of training data in order to have a desirable performance. GraphSAGE, GCN and GDPNet achieve good performances on small training data, and GDPNet make more improvements as the training data percentage increases.

\subsubsection{Convergence Analysis}

Fig.~\ref{fig:converge} shows the convergence analysis of GDPNet on Citeseer and PubMed. We initialize the policy randomly when epoch equals $0$, and the neighbors are randomly selected as signal neighbors. We observe that Citeseer converges faster than PubMed. One explanation would be that PubMed has more nodes than Citeseer, which requires more time to explore the policy for nodes.

%% file: tex/related.tex
\section{Related Work}\label{sec:related}

In this section, we briefly describe previous graph representation learning approaches including matrix factorization based methods and graph neural network based methods, and recent advancements in applying reinforcement learning on graph.

\subsection{Graph Representation Learning}
Graph representation learning tries to encode the graph structure information into vector representations. The main idea is to learn a mapping function from the nodes or entire graphs into an embedding space where the geometric relationships in the low-dimensional space coincide with the original graph. The methods can be grouped into two categories: matrix factorization based methods and graph neural network based methods~\cite{hamilton2017representation}.

\subsubsection{Matrix Factorization based Embedding}

Matrix factorization based methods learns an embedding look-up table which trains unique embedding vectors for each node independently.  These methods largely focused on matrix-factorization approaches and random walk approaches~\cite{hamilton2017representation,goyal2018graph,cai2018comprehensive,cui2018survey}. Matrix-factorization approaches utilize dimension reduction methodology to learn the representations~\cite{cao2015grarep,ahmed2013distributed,ou2016asymmetric}  with the loss of node pair similarity. Inspired by the success of natural language processing~\cite{mikolov2013distributed}, a set of methods use random walk to learn the node embeddings where the node similarity is calculated by co-occurrence statistics from sentence-like vertex sequences generated by random walks among connected vertices~\cite{perozzi2014deepwalk,yang2015network,grover2016node2vec,tang2015line,dong2017metapath2vec}. The random-walk based method have been verified to be unified into the matrix factorization framework~\cite{qiu2018network}.

\subsubsection{Graph Neural Network based Embedding}
A set of graph neural network based embedding methods are proposed recently for representation learning~\cite{bruna2013spectral,duvenaud2015convolutional,li2015gated,niepert2016learning}. GCN~\cite{kipf2016semi} first proposes the first-order graph convolution layer to perform recursive neighborhood
aggregation based on the local connection.  Instead of utilizing full graph Laplacian during training in the GCN, GraphSAGE~\cite{hamilton2017inductive} considers the inductive setting to handle the large scale graph with batch training and neighborhood sampling. Followed by GraphSAGE, self-attention mechanism has been explored to enhance the representation learning performance~\cite{velivckovic2017graph,zhang2018gaan}. To accelerate the training of GCNs, ~\cite{chen2018fastgcn} samples the nodes in each layer independently, while~\cite{huang2018adaptive} samples the lower layer conditioned on the top one and the sampled neighborhoods are shared by different parent node. In this work, we propose to find an effective subset of neighbors for learning robust representations.

\subsection{Reinforcement Learning on Graph}

Reinforcement learning solves the sequential decision making problem with the goal of maximizing cumulative rewards of these decisions. A set of work used reinforcement learning to solve the sequential decision making problems in graph, such as minimum vertex cover, maximum cut and travelling salesman problem~\cite{daihan,bello2016neural}. You et al.~\cite{you2018graph} considered the molecular graph generation process as a sequential decision making process where the reward function is designed by non-differentiable rules. Dai et al.~\cite{dai2018adversarial} utilized reinforcement learning to learn an attack policy to make multiple decisions (delete or add edges in the graph) to attack the graph. 

%% file: tex/conclusion.tex
\section{Conclusion}\label{sec:con}

In this paper, we developed a novel framework, GDPNet, to learn robust representations from noisy graph data through reinforcement learning. GDPNet includes two phases: \emph{signal neighbor selection} and \emph{representation learning}. It learns a policy to sequentially select the \emph{signal} neighbors for each node, and then aggregates the information from the selected neighbors to learn node representations for the down-stream tasks. These two learning phases are complementary and achieves significant improvement. 
We show that our method mathematically is equivalent to maximizing the submodular function with the carefully designed reward function, which guarantees our objective value can be bounded by $(1-\frac{1}{e}) R(\N(v)^*)$. Note that GDPNet is naturally an \emph{inductive} model which can generate representations for unseen nodes. 
Experiments on a set of well-studied datasets provide empirical evidence for our analytical results, and yield significant gains in performance over state-of-the-art baselines. 



%% file: main.bbl
\begin{thebibliography}{10}

\bibitem{hamilton2017representation}
W.~L. Hamilton, R.~Ying, and J.~Leskovec, ``Representation learning on graphs:
  Methods and applications,'' {\em arXiv preprint arXiv:1709.05584}, 2017.

\bibitem{cai2018comprehensive}
H.~Cai, V.~W. Zheng, and K.~C.-C. Chang, ``A comprehensive survey of graph
  embedding: Problems, techniques, and applications,'' {\em TKDE}, vol.~30,
  no.~9, pp.~1616--1637, 2018.

\bibitem{cui2018survey}
P.~Cui, X.~Wang, J.~Pei, and W.~Zhu, ``A survey on network embedding,'' {\em
  IEEE Transactions on Knowledge and Data Engineering}, 2018.

\bibitem{yu2018learning}
W.~Yu, C.~Zheng, W.~Cheng, C.~C. Aggarwal, D.~Song, B.~Zong, H.~Chen, and
  W.~Wang, ``Learning deep network representations with adversarially
  regularized autoencoders,'' in {\em KDD}, pp.~2663--2671, ACM, 2018.

\bibitem{scarselli2008graph}
F.~Scarselli, M.~Gori, A.~C. Tsoi, M.~Hagenbuchner, and G.~Monfardini, ``The
  graph neural network model,'' {\em IEEE Transactions on Neural Networks},
  vol.~20, no.~1, pp.~61--80, 2008.

\bibitem{zhou2018graph}
J.~Zhou, G.~Cui, Z.~Zhang, C.~Yang, Z.~Liu, and M.~Sun, ``Graph neural
  networks: A review of methods and applications,'' {\em arXiv preprint
  arXiv:1812.08434}, 2018.

\bibitem{hamilton2017inductive}
W.~Hamilton, Z.~Ying, and J.~Leskovec, ``Inductive representation learning on
  large graphs,'' in {\em NeurIPS}, pp.~1024--1034, 2017.

\bibitem{kipf2016semi}
T.~N. Kipf and M.~Welling, ``Semi-supervised classification with graph
  convolutional networks,'' {\em arXiv preprint arXiv:1609.02907}, 2016.

\bibitem{chen2018fastgcn}
J.~Chen, T.~Ma, and C.~Xiao, ``Fastgcn: fast learning with graph convolutional
  networks via importance sampling,'' {\em arXiv preprint arXiv:1801.10247},
  2018.

\bibitem{velivckovic2017graph}
P.~Veli{\v{c}}kovi{\'c}, G.~Cucurull, A.~Casanova, A.~Romero, P.~Lio, and
  Y.~Bengio, ``Graph attention networks,'' {\em arXiv preprint
  arXiv:1710.10903}, 2017.

\bibitem{sutton2000policy}
R.~S. Sutton, D.~A. McAllester, S.~P. Singh, and Y.~Mansour, ``Policy gradient
  methods for reinforcement learning with function approximation,'' in {\em
  NeurIPS}, pp.~1057--1063, 2000.

\bibitem{liu2015optimality}
W.~Liu and I.~Tsang, ``On the optimality of classifier chain for multi-label
  classification,'' in {\em NeurIPS}, pp.~712--720, 2015.

\bibitem{nemhauser1978analysis}
G.~L. Nemhauser, L.~A. Wolsey, and M.~L. Fisher, ``An analysis of
  approximations for maximizing submodular set functions—i,'' {\em
  Mathematical programming}, vol.~14, no.~1, pp.~265--294, 1978.

\bibitem{schulman2017proximal}
J.~Schulman, F.~Wolski, P.~Dhariwal, A.~Radford, and O.~Klimov, ``Proximal
  policy optimization algorithms,'' {\em arXiv preprint arXiv:1707.06347},
  2017.

\bibitem{maaten2008visualizing}
L.~v.~d. Maaten and G.~Hinton, ``Visualizing data using t-sne,'' {\em Journal
  of machine learning research}, vol.~9, no.~Nov, pp.~2579--2605, 2008.

\bibitem{goyal2018graph}
P.~Goyal and E.~Ferrara, ``Graph embedding techniques, applications, and
  performance: A survey,'' {\em Knowledge-Based Systems}, vol.~151, pp.~78--94,
  2018.

\bibitem{cao2015grarep}
S.~Cao, W.~Lu, and Q.~Xu, ``Grarep: Learning graph representations with global
  structural information,'' in {\em Proceedings of the 24th ACM international
  on conference on information and knowledge management}, pp.~891--900, ACM,
  2015.

\bibitem{ahmed2013distributed}
A.~Ahmed, N.~Shervashidze, S.~Narayanamurthy, V.~Josifovski, and A.~J. Smola,
  ``Distributed large-scale natural graph factorization,'' in {\em WWW},
  pp.~37--48, ACM, 2013.

\bibitem{ou2016asymmetric}
M.~Ou, P.~Cui, J.~Pei, Z.~Zhang, and W.~Zhu, ``Asymmetric transitivity
  preserving graph embedding,'' in {\em KDD}, pp.~1105--1114, ACM, 2016.

\bibitem{mikolov2013distributed}
T.~Mikolov, I.~Sutskever, K.~Chen, G.~S. Corrado, and J.~Dean, ``Distributed
  representations of words and phrases and their compositionality,'' in {\em
  NeurIPS}, pp.~3111--3119, 2013.

\bibitem{perozzi2014deepwalk}
B.~Perozzi, R.~Al-Rfou, and S.~Skiena, ``Deepwalk: Online learning of social
  representations,'' in {\em KDD}, pp.~701--710, ACM, 2014.

\bibitem{yang2015network}
C.~Yang, Z.~Liu, D.~Zhao, M.~Sun, and E.~Chang, ``Network representation
  learning with rich text information,'' in {\em IJCAI}, 2015.

\bibitem{grover2016node2vec}
A.~Grover and J.~Leskovec, ``node2vec: Scalable feature learning for
  networks,'' in {\em KDD}, pp.~855--864, ACM, 2016.

\bibitem{tang2015line}
J.~Tang, M.~Qu, M.~Wang, M.~Zhang, J.~Yan, and Q.~Mei, ``Line: Large-scale
  information network embedding,'' in {\em WWW}, pp.~1067--1077, International
  World Wide Web Conferences Steering Committee, 2015.

\bibitem{dong2017metapath2vec}
Y.~Dong, N.~V. Chawla, and A.~Swami, ``metapath2vec: Scalable representation
  learning for heterogeneous networks,'' in {\em KDD}, pp.~135--144, ACM, 2017.

\bibitem{qiu2018network}
J.~Qiu, Y.~Dong, H.~Ma, J.~Li, K.~Wang, and J.~Tang, ``Network embedding as
  matrix factorization: Unifying deepwalk, line, pte, and node2vec,'' in {\em
  Proceedings of the Eleventh ACM International Conference on Web Search and
  Data Mining}, pp.~459--467, ACM, 2018.

\bibitem{bruna2013spectral}
J.~Bruna, W.~Zaremba, A.~Szlam, and Y.~LeCun, ``Spectral networks and locally
  connected networks on graphs,'' in {\em ICLR}, 2015.

\bibitem{duvenaud2015convolutional}
D.~K. Duvenaud, D.~Maclaurin, J.~Iparraguirre, R.~Bombarell, T.~Hirzel,
  A.~Aspuru-Guzik, and R.~P. Adams, ``Convolutional networks on graphs for
  learning molecular fingerprints,'' in {\em NeurIPS}, pp.~2224--2232, 2015.

\bibitem{li2015gated}
Y.~Li, D.~Tarlow, M.~Brockschmidt, and R.~Zemel, ``Gated graph sequence neural
  networks,'' {\em arXiv preprint arXiv:1511.05493}, 2015.

\bibitem{niepert2016learning}
M.~Niepert, M.~Ahmed, and K.~Kutzkov, ``Learning convolutional neural networks
  for graphs,'' in {\em ICML}, pp.~2014--2023, 2016.

\bibitem{zhang2018gaan}
J.~Zhang, X.~Shi, J.~Xie, H.~Ma, I.~King, and D.-Y. Yeung, ``Gaan: Gated
  attention networks for learning on large and spatiotemporal graphs,'' {\em
  arXiv preprint arXiv:1803.07294}, 2018.

\bibitem{huang2018adaptive}
W.~Huang, T.~Zhang, Y.~Rong, and J.~Huang, ``Adaptive sampling towards fast
  graph representation learning,'' in {\em NeurIPS}, pp.~4558--4567, 2018.

\bibitem{daihan}
E.~Khalil, H.~Dai, Y.~Zhang, B.~Dilkina, and L.~Song, ``Learning combinatorial
  optimization algorithms over graphs,'' in {\em NeurIPS}, pp.~6348--6358,
  2017.

\bibitem{bello2016neural}
I.~Bello, H.~Pham, Q.~V. Le, M.~Norouzi, and S.~Bengio, ``Neural combinatorial
  optimization with reinforcement learning,'' in {\em ICLR}, 2017.

\bibitem{you2018graph}
J.~You, B.~Liu, Z.~Ying, V.~Pande, and J.~Leskovec, ``Graph convolutional
  policy network for goal-directed molecular graph generation,'' in {\em
  NeurIPS}, pp.~6410--6421, 2018.

\bibitem{dai2018adversarial}
H.~Dai, H.~Li, T.~Tian, X.~Huang, L.~Wang, J.~Zhu, and L.~Song, ``Adversarial
  attack on graph structured data,'' in {\em ICML}, 2018.

\end{thebibliography}



\begin{thebibliography}{43}


\ifx \showCODEN    \undefined \def \showCODEN     #1{\unskip}     \fi
\ifx \showDOI      \undefined \def \showDOI       #1{#1}\fi
\ifx \showISBNx    \undefined \def \showISBNx     #1{\unskip}     \fi
\ifx \showISBNxiii \undefined \def \showISBNxiii  #1{\unskip}     \fi
\ifx \showISSN     \undefined \def \showISSN      #1{\unskip}     \fi
\ifx \showLCCN     \undefined \def \showLCCN      #1{\unskip}     \fi
\ifx \shownote     \undefined \def \shownote      #1{#1}          \fi
\ifx \showarticletitle \undefined \def \showarticletitle #1{#1}   \fi
\ifx \showURL      \undefined \def \showURL       {\relax}        \fi
\providecommand\bibfield[2]{#2}
\providecommand\bibinfo[2]{#2}
\providecommand\natexlab[1]{#1}
\providecommand\showeprint[2][]{arXiv:#2}

\bibitem[\protect\citeauthoryear{Abbeel and Ng}{Abbeel and Ng}{2004}]%
        {abbeel2004apprenticeship}
\bibfield{author}{\bibinfo{person}{Pieter Abbeel} {and}
  \bibinfo{person}{Andrew~Y Ng}.} \bibinfo{year}{2004}\natexlab{}.
\newblock \showarticletitle{Apprenticeship learning via inverse reinforcement
  learning}. In \bibinfo{booktitle}{\emph{ICML}}. ACM, \bibinfo{pages}{1}.
\newblock


\bibitem[\protect\citeauthoryear{Almirall, Compton, Gunlicks-Stoessel, Duan,
  and Murphy}{Almirall et~al\mbox{.}}{2012}]%
        {Almirall2012Designing}
\bibfield{author}{\bibinfo{person}{D Almirall}, \bibinfo{person}{S.~N.
  Compton}, \bibinfo{person}{M Gunlicks-Stoessel}, \bibinfo{person}{N. Duan},
  {and} \bibinfo{person}{S.~A. Murphy}.} \bibinfo{year}{2012}\natexlab{}.
\newblock \showarticletitle{Designing a pilot sequential multiple assignment
  randomized trial for developing an adaptive treatment strategy}.
\newblock \bibinfo{journal}{\emph{Statistics in Medicine}}
  (\bibinfo{year}{2012}), \bibinfo{pages}{1887--902}.
\newblock


\bibitem[\protect\citeauthoryear{Bajor and Lasko}{Bajor and Lasko}{2017}]%
        {bajor2017predicting}
\bibfield{author}{\bibinfo{person}{Jacek~M Bajor} {and}
  \bibinfo{person}{Thomas~A Lasko}.} \bibinfo{year}{2017}\natexlab{}.
\newblock \showarticletitle{Predicting Medications from Diagnostic Codes with
  Recurrent Neural Networks}.
\newblock \bibinfo{journal}{\emph{ICLR}} (\bibinfo{year}{2017}).
\newblock


\bibitem[\protect\citeauthoryear{Barto}{Barto}{2002}]%
        {Barto2002Reinforcement}
\bibfield{author}{\bibinfo{person}{Andrew~G Barto}.}
  \bibinfo{year}{2002}\natexlab{}.
\newblock \showarticletitle{Reinforcement Learning in Motor Control}. In
  \bibinfo{booktitle}{\emph{The handbook of brain theory and neural networks}}.
\newblock


\bibitem[\protect\citeauthoryear{Barto}{Barto}{2004}]%
        {barto2004j}
\bibfield{author}{\bibinfo{person}{MTRAG Barto}.}
  \bibinfo{year}{2004}\natexlab{}.
\newblock \showarticletitle{supervised actor-critic reinforcement learning}.
\newblock \bibinfo{journal}{\emph{Handbook of learning and approximate dynamic
  programming}} (\bibinfo{year}{2004}), \bibinfo{pages}{359}.
\newblock


\bibitem[\protect\citeauthoryear{Benbrahim and Franklin}{Benbrahim and
  Franklin}{1997}]%
        {benbrahim1997biped}
\bibfield{author}{\bibinfo{person}{Hamid Benbrahim} {and}
  \bibinfo{person}{Judy~A Franklin}.} \bibinfo{year}{1997}\natexlab{}.
\newblock \showarticletitle{Biped dynamic walking using reinforcement
  learning}.
\newblock \bibinfo{journal}{\emph{Robotics and Autonomous Systems}}
  (\bibinfo{year}{1997}), \bibinfo{pages}{283--302}.
\newblock


\bibitem[\protect\citeauthoryear{Chakraborty and Moodie}{Chakraborty and
  Moodie}{2013}]%
        {chakraborty2013statistical}
\bibfield{author}{\bibinfo{person}{Bibhas Chakraborty} {and}
  \bibinfo{person}{EE Moodie}.} \bibinfo{year}{2013}\natexlab{}.
\newblock \bibinfo{booktitle}{\emph{Statistical methods for dynamic treatment
  regimes}}.
\newblock \bibinfo{publisher}{Springer}.
\newblock


\bibitem[\protect\citeauthoryear{Cheerla and Gevaert}{Cheerla and
  Gevaert}{2017}]%
        {Cheerla2017MicroRNA}
\bibfield{author}{\bibinfo{person}{N Cheerla} {and} \bibinfo{person}{O
  Gevaert}.} \bibinfo{year}{2017}\natexlab{}.
\newblock \showarticletitle{MicroRNA based Pan-Cancer Diagnosis and Treatment
  Recommendation:}.
\newblock \bibinfo{journal}{\emph{Bmc Bioinformatics}} (\bibinfo{year}{2017}),
  \bibinfo{pages}{32}.
\newblock


\bibitem[\protect\citeauthoryear{Choi, Bahadori, Schuetz, Stewart, and
  Sun}{Choi et~al\mbox{.}}{2016}]%
        {pmlr-v56-Choi16}
\bibfield{author}{\bibinfo{person}{Edward Choi}, \bibinfo{person}{Mohammad~Taha
  Bahadori}, \bibinfo{person}{Andy Schuetz}, \bibinfo{person}{Walter~F.
  Stewart}, {and} \bibinfo{person}{Jimeng Sun}.}
  \bibinfo{year}{2016}\natexlab{}.
\newblock \showarticletitle{Doctor AI: Predicting Clinical Events via Recurrent
  Neural Networks}. In \bibinfo{booktitle}{\emph{Proceedings of the 1st Machine
  Learning for Healthcare Conference}}. \bibinfo{publisher}{PMLR},
  \bibinfo{pages}{301--318}.
\newblock


\bibitem[\protect\citeauthoryear{Clouse and Utgoff}{Clouse and Utgoff}{1992}]%
        {Clouse1992A}
\bibfield{author}{\bibinfo{person}{Jeffery~A. Clouse} {and}
  \bibinfo{person}{Paul~E. Utgoff}.} \bibinfo{year}{1992}\natexlab{}.
\newblock \showarticletitle{A Teaching Method for Reinforcement Learning}. In
  \bibinfo{booktitle}{\emph{International Workshop on Machine Learning}}.
  \bibinfo{pages}{92--110}.
\newblock


\bibitem[\protect\citeauthoryear{Degris, Pilarski, and Sutton}{Degris
  et~al\mbox{.}}{2012}]%
        {degris2012model}
\bibfield{author}{\bibinfo{person}{Thomas Degris}, \bibinfo{person}{Patrick~M
  Pilarski}, {and} \bibinfo{person}{Richard~S Sutton}.}
  \bibinfo{year}{2012}\natexlab{}.
\newblock \showarticletitle{Model-free reinforcement learning with continuous
  action in practice}. In \bibinfo{booktitle}{\emph{American Control Conference
  (ACC), 2012}}. IEEE, \bibinfo{pages}{2177--2182}.
\newblock


\bibitem[\protect\citeauthoryear{Finn, Levine, and Abbeel}{Finn
  et~al\mbox{.}}{2016}]%
        {finn2016guided}
\bibfield{author}{\bibinfo{person}{Chelsea Finn}, \bibinfo{person}{Sergey
  Levine}, {and} \bibinfo{person}{Pieter Abbeel}.}
  \bibinfo{year}{2016}\natexlab{}.
\newblock \showarticletitle{Guided cost learning: Deep inverse optimal control
  via policy optimization}. In \bibinfo{booktitle}{\emph{International
  Conference on Machine Learning}}. \bibinfo{pages}{49--58}.
\newblock


\bibitem[\protect\citeauthoryear{Gunlicksstoessel, Mufson, Westervelt,
  Almirall, and Murphy}{Gunlicksstoessel et~al\mbox{.}}{2017}]%
        {Gunlicksstoessel2017A}
\bibfield{author}{\bibinfo{person}{M Gunlicksstoessel}, \bibinfo{person}{L
  Mufson}, \bibinfo{person}{A Westervelt}, \bibinfo{person}{D Almirall}, {and}
  \bibinfo{person}{S Murphy}.} \bibinfo{year}{2017}\natexlab{}.
\newblock \showarticletitle{A Pilot SMART for Developing an Adaptive Treatment
  Strategy for Adolescent Depression.}
\newblock \bibinfo{journal}{\emph{J Clin Child Adolesc Psychol}}
  (\bibinfo{year}{2017}), \bibinfo{pages}{1--15}.
\newblock


\bibitem[\protect\citeauthoryear{Hausknecht and Stone}{Hausknecht and
  Stone}{2015}]%
        {hausknecht2015deep}
\bibfield{author}{\bibinfo{person}{Matthew Hausknecht} {and}
  \bibinfo{person}{Peter Stone}.} \bibinfo{year}{2015}\natexlab{}.
\newblock \showarticletitle{Deep recurrent q-learning for partially observable
  mdps}.
\newblock \bibinfo{journal}{\emph{CoRR, abs/1507.06527}}
  (\bibinfo{year}{2015}).
\newblock


\bibitem[\protect\citeauthoryear{Hu, Perer, and Wang}{Hu et~al\mbox{.}}{2016}]%
        {hu2016data}
\bibfield{author}{\bibinfo{person}{Jianying Hu}, \bibinfo{person}{Adam Perer},
  {and} \bibinfo{person}{Fei Wang}.} \bibinfo{year}{2016}\natexlab{}.
\newblock \showarticletitle{Data driven analytics for personalized healthcare}.
\newblock In \bibinfo{booktitle}{\emph{Healthcare Information Management
  Systems}}. \bibinfo{pages}{529--554}.
\newblock


\bibitem[\protect\citeauthoryear{Johnson, Pollard, Shen, Lehman, Feng,
  Ghassemi, Moody, Szolovits, Celi, and Mark}{Johnson et~al\mbox{.}}{2016}]%
        {johnson2016mimic}
\bibfield{author}{\bibinfo{person}{Alistair~EW Johnson}, \bibinfo{person}{Tom~J
  Pollard}, \bibinfo{person}{Lu Shen}, \bibinfo{person}{Li-wei~H Lehman},
  \bibinfo{person}{Mengling Feng}, \bibinfo{person}{Mohammad Ghassemi},
  \bibinfo{person}{Benjamin Moody}, \bibinfo{person}{Peter Szolovits},
  \bibinfo{person}{Leo~Anthony Celi}, {and} \bibinfo{person}{Roger~G Mark}.}
  \bibinfo{year}{2016}\natexlab{}.
\newblock \showarticletitle{MIMIC-III, a freely accessible critical care
  database}.
\newblock \bibinfo{journal}{\emph{Scientific data}} (\bibinfo{year}{2016}).
\newblock


\bibitem[\protect\citeauthoryear{Konda and Tsitsiklis}{Konda and
  Tsitsiklis}{2000}]%
        {konda2000actor}
\bibfield{author}{\bibinfo{person}{Vijay~R Konda} {and} \bibinfo{person}{John~N
  Tsitsiklis}.} \bibinfo{year}{2000}\natexlab{}.
\newblock \showarticletitle{Actor-critic algorithms}. In
  \bibinfo{booktitle}{\emph{Advances in neural information processing
  systems}}. \bibinfo{pages}{1008--1014}.
\newblock


\bibitem[\protect\citeauthoryear{Levine, Finn, Darrell, and Abbeel}{Levine
  et~al\mbox{.}}{2016}]%
        {levine2016end}
\bibfield{author}{\bibinfo{person}{Sergey Levine}, \bibinfo{person}{Chelsea
  Finn}, \bibinfo{person}{Trevor Darrell}, {and} \bibinfo{person}{Pieter
  Abbeel}.} \bibinfo{year}{2016}\natexlab{}.
\newblock \showarticletitle{End-to-end training of deep visuomotor policies}.
\newblock \bibinfo{journal}{\emph{Journal of Machine Learning Research}}
  (\bibinfo{year}{2016}), \bibinfo{pages}{1--40}.
\newblock


\bibitem[\protect\citeauthoryear{Levine, Popovic, and Koltun}{Levine
  et~al\mbox{.}}{2011}]%
        {levine2011nonlinear}
\bibfield{author}{\bibinfo{person}{Sergey Levine}, \bibinfo{person}{Zoran
  Popovic}, {and} \bibinfo{person}{Vladlen Koltun}.}
  \bibinfo{year}{2011}\natexlab{}.
\newblock \showarticletitle{Nonlinear inverse reinforcement learning with
  gaussian processes}. In \bibinfo{booktitle}{\emph{Advances in Neural
  Information Processing Systems}}. \bibinfo{pages}{19--27}.
\newblock


\bibitem[\protect\citeauthoryear{Lillicrap, Hunt, Pritzel, Heess, Erez, Tassa,
  Silver, and Wierstra}{Lillicrap et~al\mbox{.}}{2015}]%
        {lillicrap2015continuous}
\bibfield{author}{\bibinfo{person}{Timothy~P Lillicrap},
  \bibinfo{person}{Jonathan~J Hunt}, \bibinfo{person}{Alexander Pritzel},
  \bibinfo{person}{Nicolas Heess}, \bibinfo{person}{Tom Erez},
  \bibinfo{person}{Yuval Tassa}, \bibinfo{person}{David Silver}, {and}
  \bibinfo{person}{Daan Wierstra}.} \bibinfo{year}{2015}\natexlab{}.
\newblock \showarticletitle{Continuous control with deep reinforcement
  learning}.
\newblock \bibinfo{journal}{\emph{arXiv preprint arXiv:1509.02971}}
  (\bibinfo{year}{2015}).
\newblock


\bibitem[\protect\citeauthoryear{Marik}{Marik}{2015}]%
        {Marik2015The}
\bibfield{author}{\bibinfo{person}{P.~E. Marik}.}
  \bibinfo{year}{2015}\natexlab{}.
\newblock \showarticletitle{The demise of early goal-directed therapy for
  severe sepsis and septic shock.}
\newblock \bibinfo{journal}{\emph{Acta Anaesthesiologica Scandinavica}}
  (\bibinfo{year}{2015}), \bibinfo{pages}{561}.
\newblock


\bibitem[\protect\citeauthoryear{Mihatsch and Neuneier}{Mihatsch and
  Neuneier}{2002}]%
        {mihatsch2002risk}
\bibfield{author}{\bibinfo{person}{Oliver Mihatsch} {and}
  \bibinfo{person}{Ralph Neuneier}.} \bibinfo{year}{2002}\natexlab{}.
\newblock \showarticletitle{Risk-sensitive reinforcement learning}.
\newblock \bibinfo{journal}{\emph{Machine learning}} (\bibinfo{year}{2002}),
  \bibinfo{pages}{267--290}.
\newblock


\bibitem[\protect\citeauthoryear{Mnih, Kavukcuoglu, Silver, Rusu, Veness,
  Bellemare, Graves, Riedmiller, Fidjeland, Ostrovski, et~al\mbox{.}}{Mnih
  et~al\mbox{.}}{2015}]%
        {mnih2015human}
\bibfield{author}{\bibinfo{person}{Volodymyr Mnih}, \bibinfo{person}{Koray
  Kavukcuoglu}, \bibinfo{person}{David Silver}, \bibinfo{person}{Andrei~A
  Rusu}, \bibinfo{person}{Joel Veness}, \bibinfo{person}{Marc~G Bellemare},
  \bibinfo{person}{Alex Graves}, \bibinfo{person}{Martin Riedmiller},
  \bibinfo{person}{Andreas~K Fidjeland}, \bibinfo{person}{Georg Ostrovski},
  {et~al\mbox{.}}} \bibinfo{year}{2015}\natexlab{}.
\newblock \showarticletitle{Human-level control through deep reinforcement
  learning}.
\newblock \bibinfo{journal}{\emph{Nature}} (\bibinfo{year}{2015}),
  \bibinfo{pages}{529--533}.
\newblock


\bibitem[\protect\citeauthoryear{Murphy}{Murphy}{2003}]%
        {murphy2003optimal}
\bibfield{author}{\bibinfo{person}{Susan~A Murphy}.}
  \bibinfo{year}{2003}\natexlab{}.
\newblock \showarticletitle{Optimal dynamic treatment regimes}.
\newblock \bibinfo{journal}{\emph{Journal of the Royal Statistical Society:
  Series B (Statistical Methodology)}} \bibinfo{volume}{65},
  \bibinfo{number}{2} (\bibinfo{year}{2003}), \bibinfo{pages}{331--355}.
\newblock


\bibitem[\protect\citeauthoryear{Nemati, Ghassemi, and Clifford}{Nemati
  et~al\mbox{.}}{2016}]%
        {Nemati2016Optimal}
\bibfield{author}{\bibinfo{person}{Shamim Nemati}, \bibinfo{person}{Mohammad~M.
  Ghassemi}, {and} \bibinfo{person}{Gari~D. Clifford}.}
  \bibinfo{year}{2016}\natexlab{}.
\newblock \showarticletitle{Optimal medication dosing from suboptimal clinical
  examples: A deep reinforcement learning approach}. In
  \bibinfo{booktitle}{\emph{Engineering in Medicine and Biology Society}}.
  \bibinfo{pages}{2978}.
\newblock


\bibitem[\protect\citeauthoryear{Prasad, Cheng, Chivers, Draugelis, and
  Engelhardt}{Prasad et~al\mbox{.}}{2017}]%
        {Prasad2017A}
\bibfield{author}{\bibinfo{person}{Niranjani Prasad}, \bibinfo{person}{Li~Fang
  Cheng}, \bibinfo{person}{Corey Chivers}, \bibinfo{person}{Michael Draugelis},
  {and} \bibinfo{person}{Barbara~E Engelhardt}.}
  \bibinfo{year}{2017}\natexlab{}.
\newblock \showarticletitle{A Reinforcement Learning Approach to Weaning of
  Mechanical Ventilation in Intensive Care Units}.
\newblock  (\bibinfo{year}{2017}).
\newblock


\bibitem[\protect\citeauthoryear{Raghu, Komorowski, Ahmed, Celi, Szolovits, and
  Ghassemi}{Raghu et~al\mbox{.}}{2017}]%
        {raghu2017deep}
\bibfield{author}{\bibinfo{person}{Aniruddh Raghu}, \bibinfo{person}{Matthieu
  Komorowski}, \bibinfo{person}{Imran Ahmed}, \bibinfo{person}{Leo Celi},
  \bibinfo{person}{Peter Szolovits}, {and} \bibinfo{person}{Marzyeh Ghassemi}.}
  \bibinfo{year}{2017}\natexlab{}.
\newblock \showarticletitle{Deep Reinforcement Learning for Sepsis Treatment}.
\newblock \bibinfo{journal}{\emph{arXiv preprint arXiv:1711.09602}}
  (\bibinfo{year}{2017}).
\newblock


\bibitem[\protect\citeauthoryear{Ratliff, Bagnell, and Zinkevich}{Ratliff
  et~al\mbox{.}}{2006}]%
        {ratliff2006maximum}
\bibfield{author}{\bibinfo{person}{Nathan~D Ratliff}, \bibinfo{person}{J~Andrew
  Bagnell}, {and} \bibinfo{person}{Martin~A Zinkevich}.}
  \bibinfo{year}{2006}\natexlab{}.
\newblock \showarticletitle{Maximum margin planning}. In
  \bibinfo{booktitle}{\emph{ICML}}. ACM, \bibinfo{pages}{729--736}.
\newblock


\bibitem[\protect\citeauthoryear{Robins}{Robins}{1986}]%
        {robins1986new}
\bibfield{author}{\bibinfo{person}{James Robins}.}
  \bibinfo{year}{1986}\natexlab{}.
\newblock \showarticletitle{A new approach to causal inference in mortality
  studies with a sustained exposure period—application to control of the
  healthy worker survivor effect}.
\newblock \bibinfo{journal}{\emph{Mathematical modelling}}
  (\bibinfo{year}{1986}), \bibinfo{pages}{1393--1512}.
\newblock


\bibitem[\protect\citeauthoryear{Rosen-Zvi, Altmann, Aharoni, Neuvirth,
  Sonnerborg, Schulter, Struck, Peres, Incardona, and Kaiser}{Rosen-Zvi
  et~al\mbox{.}}{2008}]%
        {Rosen2008Selecting}
\bibfield{author}{\bibinfo{person}{M Rosen-Zvi}, \bibinfo{person}{AProsperi~M
  Altmann}, \bibinfo{person}{E Aharoni}, \bibinfo{person}{H Neuvirth},
  \bibinfo{person}{A Sonnerborg}, \bibinfo{person}{E Schulter},
  \bibinfo{person}{D Struck}, \bibinfo{person}{Y Peres}, \bibinfo{person}{F
  Incardona}, {and} \bibinfo{person}{R Kaiser}.}
  \bibinfo{year}{2008}\natexlab{}.
\newblock \showarticletitle{Selecting anti-HIV therapies based on a variety of
  genomic and clinical factors}.
\newblock \bibinfo{journal}{\emph{Bioinformatics}} (\bibinfo{year}{2008}),
  \bibinfo{pages}{399--406}.
\newblock


\bibitem[\protect\citeauthoryear{Shortreed and Moodie}{Shortreed and
  Moodie}{2012}]%
        {Shortreed2012Estimating}
\bibfield{author}{\bibinfo{person}{Susan~M. Shortreed} {and}
  \bibinfo{person}{Erica E.~M. Moodie}.} \bibinfo{year}{2012}\natexlab{}.
\newblock \showarticletitle{Estimating the optimal dynamic antipsychotic
  treatment regime: evidence from the sequential multiple-assignment randomized
  Clinical Antipsychotic Trials of Intervention and Effectiveness schizophrenia
  study}.
\newblock \bibinfo{journal}{\emph{Journal of the Royal Statistical Society}}
  (\bibinfo{year}{2012}), \bibinfo{pages}{577--599}.
\newblock


\bibitem[\protect\citeauthoryear{Silver, Lever, Heess, Degris, Wierstra, and
  Riedmiller}{Silver et~al\mbox{.}}{2014}]%
        {silver2014deterministic}
\bibfield{author}{\bibinfo{person}{David Silver}, \bibinfo{person}{Guy Lever},
  \bibinfo{person}{Nicolas Heess}, \bibinfo{person}{Thomas Degris},
  \bibinfo{person}{Daan Wierstra}, {and} \bibinfo{person}{Martin Riedmiller}.}
  \bibinfo{year}{2014}\natexlab{}.
\newblock \showarticletitle{Deterministic policy gradient algorithms}. In
  \bibinfo{booktitle}{\emph{Proceedings of the 31st International Conference on
  Machine Learning (ICML-14)}}. \bibinfo{pages}{387--395}.
\newblock


\bibitem[\protect\citeauthoryear{Sun, Liu, Guo, Xiong, and Xie}{Sun
  et~al\mbox{.}}{2016}]%
        {sun2016data}
\bibfield{author}{\bibinfo{person}{Leilei Sun}, \bibinfo{person}{Chuanren Liu},
  \bibinfo{person}{Chonghui Guo}, \bibinfo{person}{Hui Xiong}, {and}
  \bibinfo{person}{Yanming Xie}.} \bibinfo{year}{2016}\natexlab{}.
\newblock \showarticletitle{Data-driven Automatic Treatment Regimen Development
  and Recommendation.}. In \bibinfo{booktitle}{\emph{KDD}}.
  \bibinfo{pages}{1865--1874}.
\newblock


\bibitem[\protect\citeauthoryear{Susan M.~Shortreed}{Susan
  M.~Shortreed}{2011}]%
        {Susan2011Informing}
\bibfield{author}{\bibinfo{person}{Daniel J. Lizotte T. Scott Stroup Joelle
  Pineau Susan A.~Murphy Susan M.~Shortreed, Eric~Laber}.}
  \bibinfo{year}{2011}\natexlab{}.
\newblock \showarticletitle{Informing sequential clinical decision-making
  through reinforcement learning: an empirical study}.
\newblock \bibinfo{journal}{\emph{Machine Learning}} (\bibinfo{year}{2011}),
  \bibinfo{pages}{109--136}.
\newblock


\bibitem[\protect\citeauthoryear{Sutton, McAllester, Singh, and Mansour}{Sutton
  et~al\mbox{.}}{2000}]%
        {sutton2000policy}
\bibfield{author}{\bibinfo{person}{Richard~S Sutton}, \bibinfo{person}{David~A
  McAllester}, \bibinfo{person}{Satinder~P Singh}, {and}
  \bibinfo{person}{Yishay Mansour}.} \bibinfo{year}{2000}\natexlab{}.
\newblock \showarticletitle{Policy gradient methods for reinforcement learning
  with function approximation}. In \bibinfo{booktitle}{\emph{Advances in neural
  information processing systems}}. \bibinfo{pages}{1057--1063}.
\newblock


\bibitem[\protect\citeauthoryear{Watkins and Dayan}{Watkins and Dayan}{1992}]%
        {watkins1992q}
\bibfield{author}{\bibinfo{person}{Christopher~JCH Watkins} {and}
  \bibinfo{person}{Peter Dayan}.} \bibinfo{year}{1992}\natexlab{}.
\newblock \showarticletitle{Q-learning}.
\newblock \bibinfo{journal}{\emph{Machine learning}} (\bibinfo{year}{1992}),
  \bibinfo{pages}{279--292}.
\newblock


\bibitem[\protect\citeauthoryear{Weng, Gao, He, Yan, and Szolovits}{Weng
  et~al\mbox{.}}{2017}]%
        {weng2017representation}
\bibfield{author}{\bibinfo{person}{Wei-Hung Weng}, \bibinfo{person}{Mingwu
  Gao}, \bibinfo{person}{Ze He}, \bibinfo{person}{Susu Yan}, {and}
  \bibinfo{person}{Peter Szolovits}.} \bibinfo{year}{2017}\natexlab{}.
\newblock \showarticletitle{Representation and Reinforcement Learning for
  Personalized Glycemic Control in Septic Patients}.
\newblock \bibinfo{journal}{\emph{arXiv preprint arXiv:1712.00654}}
  (\bibinfo{year}{2017}).
\newblock


\bibitem[\protect\citeauthoryear{Wierstra, Foerster, Peters, and
  Schmidhuber}{Wierstra et~al\mbox{.}}{2007}]%
        {wierstra2007solving}
\bibfield{author}{\bibinfo{person}{Daan Wierstra}, \bibinfo{person}{Alexander
  Foerster}, \bibinfo{person}{Jan Peters}, {and} \bibinfo{person}{Juergen
  Schmidhuber}.} \bibinfo{year}{2007}\natexlab{}.
\newblock \showarticletitle{Solving deep memory POMDPs with recurrent policy
  gradients}. In \bibinfo{booktitle}{\emph{International Conference on
  Artificial Neural Networks}}. Springer, \bibinfo{pages}{697--706}.
\newblock


\bibitem[\protect\citeauthoryear{Zhang, Wang, Hu, and Sorrentino}{Zhang
  et~al\mbox{.}}{2014}]%
        {zhang2014towards}
\bibfield{author}{\bibinfo{person}{Ping Zhang}, \bibinfo{person}{Fei Wang},
  \bibinfo{person}{Jianying Hu}, {and} \bibinfo{person}{Robert Sorrentino}.}
  \bibinfo{year}{2014}\natexlab{}.
\newblock \showarticletitle{Towards personalized medicine: leveraging patient
  similarity and drug similarity analytics}.
\newblock \bibinfo{journal}{\emph{AMIA Summits on Translational Science
  Proceedings}} (\bibinfo{year}{2014}), \bibinfo{pages}{132}.
\newblock


\bibitem[\protect\citeauthoryear{Zhang, Chen, Tang, Stewart, and Sun}{Zhang
  et~al\mbox{.}}{2017}]%
        {Zhang2017LEAP}
\bibfield{author}{\bibinfo{person}{Yutao Zhang}, \bibinfo{person}{Robert Chen},
  \bibinfo{person}{Jie Tang}, \bibinfo{person}{Walter~F. Stewart}, {and}
  \bibinfo{person}{Jimeng Sun}.} \bibinfo{year}{2017}\natexlab{}.
\newblock \showarticletitle{LEAP: Learning to Prescribe Effective and Safe
  Treatment Combinations for Multimorbidity}. In
  \bibinfo{booktitle}{\emph{KDD}}. \bibinfo{pages}{1315--1324}.
\newblock


\bibitem[\protect\citeauthoryear{Zhao, Zeng, Socinski, and Kosorok}{Zhao
  et~al\mbox{.}}{2011}]%
        {Zhao2011Reinforcement}
\bibfield{author}{\bibinfo{person}{Yufan Zhao}, \bibinfo{person}{Donglin Zeng},
  \bibinfo{person}{Mark~A Socinski}, {and} \bibinfo{person}{Michael~R
  Kosorok}.} \bibinfo{year}{2011}\natexlab{}.
\newblock \showarticletitle{Reinforcement Learning Strategies for Clinical
  Trials in Nonsmall Cell Lung Cancer}.
\newblock \bibinfo{journal}{\emph{Biometrics}} (\bibinfo{year}{2011}),
  \bibinfo{pages}{1422--1433}.
\newblock


\bibitem[\protect\citeauthoryear{Zhuo, Kyle, Elmer, Gopal, and Lakshman}{Zhuo
  et~al\mbox{.}}{2016}]%
        {Zhuo2016A}
\bibfield{author}{\bibinfo{person}{Chen Zhuo}, \bibinfo{person}{Marple Kyle},
  \bibinfo{person}{Salazar Elmer}, \bibinfo{person}{Gupta Gopal}, {and}
  \bibinfo{person}{Tamil Lakshman}.} \bibinfo{year}{2016}\natexlab{}.
\newblock \showarticletitle{A Physician Advisory System for Chronic Heart
  Failure management based on knowledge patterns}.
\newblock \bibinfo{journal}{\emph{Theory and Practice of Logic Programming}}
  (\bibinfo{year}{2016}), \bibinfo{pages}{604--618}.
\newblock


\bibitem[\protect\citeauthoryear{Ziebart, Maas, Bagnell, and Dey}{Ziebart
  et~al\mbox{.}}{2008}]%
        {ziebart2008maximum}
\bibfield{author}{\bibinfo{person}{Brian~D Ziebart}, \bibinfo{person}{Andrew~L
  Maas}, \bibinfo{person}{J~Andrew Bagnell}, {and} \bibinfo{person}{Anind~K
  Dey}.} \bibinfo{year}{2008}\natexlab{}.
\newblock \showarticletitle{Maximum Entropy Inverse Reinforcement Learning.}.
  In \bibinfo{booktitle}{\emph{AAAI}}. \bibinfo{pages}{1433--1438}.
\newblock


\end{thebibliography}
